\documentclass[twoside,11pt]{article}

\usepackage{blindtext}

\usepackage{jmlr2e}

\usepackage{amsmath,amssymb}
\usepackage{mathrsfs}
\usepackage{algorithm}
\usepackage{algorithmic}
\usepackage{bbm}

\usepackage{url}

\usepackage{parskip}

\newcommand{\signequals}{\overset{\text{S}}{=}}

\usepackage[export]{adjustbox}

\usepackage{thm-restate}

\usepackage{nicefrac}

\usepackage{graphicx}

\newcommand{\BlackBox}{\rule{1.5ex}{1.5ex}}  
\ifdefined\proof
    \renewenvironment{proof}{\par\noindent{\bf Proof\ }}{\hfill\BlackBox\\[2mm]}
\else
    \newenvironment{proof}{\par\noindent{\bf Proof\ }}{\hfill\BlackBox\\[2mm]}
\fi

\newtheorem{theorem}{Theorem}
\newtheorem{proposition}[theorem]{Proposition}
\newtheorem{lemma}[theorem]{Lemma}

\newtheorem{example}[theorem]{Example}
\newtheorem{definition}[theorem]{Definition}
\newtheorem{assumption}[theorem]{Assumption}
\newtheorem{remark}[theorem]{Remark}

\newcommand{\alglinelabel}{%
  \addtocounter{ALC@line}{-1}
  \refstepcounter{ALC@line}
  \label
}

\usepackage{selectp}

\usepackage{enumitem}
\setlist{leftmargin=0.5cm,topsep=0cm,itemsep=0cm}

\newcommand{\qdpfixedpointlambda}{\hat{\eta}^{\pi}_{\lambda}}
\newcommand{\qdpfixedpointlambdaparams}{\hat{\theta}^{\pi}_{\lambda}}
\newcommand{\indexset}{J}
\newcommand{\projmap}{\mathrm{P}_{\indexset}}
\newcommand{\cone}{\mathrm{N}}

\usepackage{lastpage}
\jmlrheading{25}{2024}{1-\pageref{LastPage}}{2/23}{5/24}{23-0154}{Mark Rowland, R\'emi Munos, Mohammad Gheshlaghi Azar, Yunhao Tang, Georg Ostrovski, Anna Harutyunyan, Karl Tuyls, Marc G. Bellemare, Will Dabney}
\ShortHeadings{An Analysis of Quantile Temporal-Difference Learning}{Rowland, Munos, Azar, Tang, Ostrovski, Harutyunyan, Tuyls, Bellemare, Dabney}
\firstpageno{1}

\begin{document}

\title{An Analysis of Quantile Temporal-Difference Learning}

\author{\name Mark Rowland \email markrowland@google.com \\
\addr Google DeepMind, London, UK
\AND
\name R\'emi Munos\\
\addr Google DeepMind, Paris, France
\AND
\name Mohammad Gheshlaghi Azar\\
\addr Google DeepMind, Seattle, USA
\AND
\name Yunhao Tang\\
\addr Google DeepMind, London, UK
\AND
\name Georg Ostrovski\\
\addr Google DeepMind, London, UK
\AND
\name Anna Harutyunyan\\
\addr Google DeepMind, London, UK
\AND
\name Karl Tuyls\\
\addr Google DeepMind, Paris, France
\AND
\name Marc G. Bellemare\\
\addr Reliant AI \& McGill University, Montr\'eal, Canada
\AND
\name Will Dabney \\
\addr Google DeepMind, Seattle, USA
}

\editor{Alexandre Proutiere}

\maketitle

\begin{abstract}%
We analyse quantile temporal-difference learning (QTD), a distributional reinforcement learning algorithm that has proven to be a key component in several successful large-scale applications of reinforcement learning.
Despite these empirical successes, a theoretical understanding of QTD has proven elusive until now.
Unlike classical TD learning, which can be analysed with standard stochastic approximation tools, QTD updates do not approximate contraction mappings, are highly non-linear, and may have multiple fixed points.
The core result of this paper is a proof of convergence to the fixed points of a related family of dynamic programming procedures with probability 1, putting QTD on firm theoretical footing.
The proof establishes connections between QTD and non-linear differential inclusions through stochastic approximation theory and non-smooth analysis.
\end{abstract}

\begin{keywords}
Reinforcement learning, temporal-difference learning, distributional reinforcement learning, stochastic approximation, differential inclusion.
\end{keywords}

\section{Introduction}

In distributional reinforcement learning, an agent aims to predict the full probability distribution over future returns it will encounter
\citep{morimura2010parametric,morimura2010nonparametric,bellemare2017distributional,bdr2022},
in contrast to predicting just the mean return, as in classical reinforcement learning \citep{sutton2018reinforcement}.
A widely-used family of algorithms for distributional reinforcement learning is based on the notion of learning \emph{quantiles} of the return distribution, an approach that originated with \citet{dabney2018distributional}, who introduced the quantile temporal-difference (QTD) learning algorithm. This approach has been particularly successful in combination with deep reinforcement learning, and has been a central component in several recent real-world applications, including sim-to-real stratospheric balloon navigation \citep{bellemare2020autonomous}, robotic manipulation \citep{bodnar2020quantile}, and algorithm discovery \citep{fawzi2022discovering}, as well as on benchmark simulated domains such as the Arcade Learning Environment \citep{bellemare13arcade,machado2018revisiting,dabney2018distributional,dabney2018implicit,yang2019fully} and racing simulation \citep{wurman2022outracing}.

Despite these empirical successes of QTD, little is known about its behaviour from a theoretical viewpoint. In particular, questions regarding the asymptotic behaviour of the algorithm (Do its predictions converge? Under what conditions? What is the qualitative character of the predictions when they do converge?) were left open. 
A core reason for this is that unlike classical TD, and other distributional reinforcement learning algorithms such as categorical temporal-difference learning \citep{rowland2018analysis,bdr2022}, the updates of QTD rely on asymmetric $L^1$ losses. As a result, these updates do not approximate the application of a contraction mapping, are highly non-linear (even in the tabular setting), and also may have multiple fixed points (depending on the exact structure of the reward distributions of the environment), and their analysis requires a distinct set of tools to those typically used to analyse temporal-difference learning algorithms.

In this paper, we prove the convergence of QTD---notably under weaker assumptions than are required in typical proofs of convergence for classical TD learning---establishing it as a sound algorithm with theoretical convergence guarantees, and paving the way for further analysis and investigation. The more general conditions stem from the structure of the QTD updates (namely, their boundedness), and the proof is obtained through the use of stochastic approximation theory with differential inclusions.

We begin by providing background on Markov decision processes, classical TD learning, and quantile regression in Section~\ref{sec:background}.
After motivating the QTD algorithm in Section~\ref{sec:qtd-and-qdp}, we describe the related family of quantile dynamic programming (QDP) algorithms, and provide a convergence analysis of these algorithms in Section~\ref{sec:dp-analysis}.
We then present the main result, a convergence analysis of QTD, in Section~\ref{sec:qtd-convergence}. The proof relies on the stochastic approximation framework set out by \citet{benaim2005stochastic}, arguing that the QTD algorithm approximates a continuous-time differential inclusion, and then constructing a Lyapunov function to demonstrate that the limiting behaviour of trajectories of the differential inclusion matches that of the QDP algorithms introduced earlier. Finally, in Section~\ref{sec:fixed-point-analysis}, we analyse the limit points of QTD, bounding their approximation error to the true return distributions of interest, and investigating the kinds of approximation artefacts that arise empirically.

\section{Background}\label{sec:background}

We first introduce background concepts and notation.

\subsection{Markov Decision Processes}

We consider a Markov decision process specified by finite state and action spaces $\mathcal{X}$ and  $\mathcal{A}$, transition kernel $P_\mathcal{X} : \mathcal{X} \times \mathcal{A} \rightarrow \mathscr{P}(\mathcal{X})$, reward distribution function $P_\mathcal{R} : \mathcal{X} \times \mathcal{A} \rightarrow \mathscr{P}_1(\mathbb{R})$, and discount factor $\gamma \in [0, 1)$. Here, $\mathscr{P}(\mathcal{X})$ is the set of probability distributions over the finite set $\mathcal{X}$, and $\mathscr{P}_1(\mathbb{R})$ is the set of probability distributions over $\mathbb{R}$ (with its usual Borel $\sigma$-algebra) with finite mean.

Given a policy $\pi : \mathcal{X} \rightarrow \mathscr{P}(\mathcal{A})$ and an initial state $x_0 \in \mathcal{X}$, an agent interacting with the environment using the policy $\pi$ generates a sequence of states, actions and rewards $(X_t, A_t, R_t)_{t =0}^\infty$, called a \emph{trajectory}, the joint distribution of which is determined by the transition dynamics and reward distributions of the environment, and the policy of the agent. More precisely, we have
\begin{itemize}
    \item $X_0 = x_0$, and for each $t \geq 0$:
    \item $A_t \mid X_{0:t}, A_{0:t-1}, R_{0:t-1} \sim \pi(\cdot|X_t)$;
    \item $R_t \mid X_{0:t}, A_{0:t}, R_{0:t-1} \sim P_\mathcal{R}(\cdot|X_t, A_t)$;
    \item $X_{t+1} \mid X_{0:t}, A_{0:t}, R_{0:t} \sim P_\mathcal{X}(\cdot|X_t, A_t)$.
\end{itemize}
The distribution of the trajectory is thus parametrised by the initial state $x_0$, and the policy $\pi$. To illustrate this dependency, we use the notation $\mathbb{P}^\pi_{x_0}$ and $\mathbb{E}^\pi_{x_0}$ to denote the probability distribution and expectation operator corresponding to this distribution, and will write $P^\pi(\cdot|x)$ for the joint distribution over a reward--next-state pair when the current state is $x$.

\subsection{Predicting Expected Returns and the Return Distribution}

The quality of the agent's performance on the trajectory is quantified by the \emph{discounted return}, or simply the \emph{return}, given by
\begin{align}\label{eq:random-return}
    \sum_{t =0}^\infty \gamma^t R_t \, .
\end{align}
The return is a random variable, whose sources of randomness are the random selections of actions made according to $\pi$, the randomness in state transitions, and the randomness in rewards observed. Typically in reinforcement learning, a single scalar summary of performance is given by the expectation of this return over all these sources of randomness. For a given policy, this is summarised across each possible starting state via the \emph{value function} $V^\pi : \mathcal{X} \rightarrow \mathbb{R}$, defined by
\begin{align}\label{eq:vpi}
    V^\pi(x) = \mathbb{E}^\pi_x\Bigg\lbrack \sum_{t =0}^\infty \gamma^t R_t \Bigg\rbrack \, .
\end{align}
Learning the value function of a policy $\pi$ from sampled trajectories generated through interaction with the environment is a central problem in reinforcement learning, referred to as the \emph{policy evaluation task}.

Each expected return is a scalar summary of a much more rich, complex object: the probability distributions of the random return in Equation~\eqref{eq:random-return} itself. \emph{Distributional reinforcement learning} \citep{bdr2022} is concerned with the problem of learning to predict the \emph{probability distribution} over returns, in contrast to just their expected value. Mathematically, the goal is to learn the return-distribution function $\eta^\pi : \mathcal{X} \rightarrow \mathscr{P}(\mathbb{R})$; for each state $x \in \mathcal{X}$, $\eta^\pi(x)$ is the probability distribution of the random return in Expression~\eqref{eq:random-return} when the trajectory begins at state $x$, and the agent acts using policy $\pi$. Mathematically, we have
\begin{align*}
    \eta^\pi(x) = \mathcal{D}^\pi_x\Bigg(
    \sum_{t \geq 0} \gamma^t R_t 
    \Bigg) \, ,
\end{align*}
where $\mathcal{D}^\pi_x$ extract the probability distribution of a random variable under $\mathbb{P}^\pi_x$.

There are several distinct motivations for aiming to learn these more complex objects. First, the richness of the distribution provides an abundance of signal for an agent to learn from, in contrast to a single scalar expectation. The strong performance of deep reinforcement learning agents that incorporate distributional predictions is hypothesised to be related to this fact \citep{dabney2018distributional,barth2018distributed,dabney2018implicit,yang2019fully}. Second, learning about the full probability distribution of returns makes possible the use of \emph{risk-sensitive} performance criteria; one may be interested in not only the expected return under a policy, but also the variance of the return, or the probability of the return being under a certain threshold.

Unlike the value function $V^\pi$, which is an element of $\mathbb{R}^{\mathcal{X}}$, and can therefore be straightforwardly represented on a computer (up to floating-point precision), the return-distribution function $\eta^\pi$ is not representable. Each object $\eta^\pi(x)$ is a probability distribution over the real numbers, and, informally speaking, probability distributions have infinitely many degrees of freedom. 
Distributional reinforcement learning algorithms therefore typically work with a subset of distributions that \emph{are} amenable to parametrisation on a computer \citep{bdr2022}. Common choices of subsets include categorical distributions \citep{bellemare2017distributional}, exponential families \citep{morimura2010parametric}, and mixtures of Gaussian distributions \citep{barth2018distributed}. Quantile temporal-difference learning, the core algorithm of study in this paper, aims to learn a particular set of quantiles of the return distribution, as described in Section~\ref{sec:qtd-and-qdp}.

\subsection{Monte Carlo and Temporal-Difference Learning}

To foreshadow our description and motivation of quantile temporal-difference learning, we recall a line of thinking that interprets the classical TD learning update rule as an approximation to Monte Carlo learning; this material is common to many introductory texts on reinforcement learning \citep{sutton2018reinforcement}, and we present it here to make a direct analogy with QTD. First, we may observe that, under the condition that all reward distributions have finite variance, $V^\pi(x)$ is the unique minimiser of the following loss function over $u \in \mathbb{R}$, the prediction of mean return at $x$:
\begin{align*}
    \mathcal{L}^\pi_x(u) = \frac{1}{2}\mathbb{E}^\pi_x\Big[\Big(u - \sum_{t=0}^\infty \gamma^t R_t\Big)^2\Big] \, .
\end{align*}
This well-known characterisation of the expectation of a random variable is readily verified by, for example, observing that the loss is convex and differentiable in $u$, and solving the equation $\partial_u \mathcal{L}^\pi_x(u) = 0$. This motivates an approach to learning $V^\pi(x)$ based on stochastic gradient descent on the loss function $\mathcal{L}^\pi_x$. We maintain an estimate  $V \in \mathbb{R}^{\mathcal{X}}$ of the value function, and each time a trajectory $(X_t, A_t, R_t)_{t \geq 0}$ beginning at state $x$ is observed, we can obtain an unbiased estimator of the negative gradient of $\mathcal{L}^\pi_x(V(x))$ as
\begin{align*}
    \sum_{t=0}^\infty \gamma^t R_t - V(x) \, ,
\end{align*}
and update $V(x)$ by taking a step in the direction of this negative gradient, with some step size $\alpha$:
\begin{align}\label{eq:mc-update}
    V(x) \leftarrow V(x) + \alpha \Big(\sum_{t=0}^\infty \gamma^t R_t - V(x)\Big) \, .
\end{align}

This is a \emph{Monte Carlo} algorithm, so called because it uses Monte Carlo samples of the random return to update the estimate $V$.

A popular alternative to this Monte Carlo algorithm is temporal-difference learning, which replaces samples from the random return with a \emph{bootstrapped} approximation to the return, obtained from a transition $(x, A, R, X')$ by combining the immediate reward $R$ with the current estimate of the expected return obtained at $X'$, resulting in the return estimate 
\begin{align}\label{eq:td-return-estimate}
    R + \gamma V(X') \, ,
\end{align}
and the corresponding update rule
\begin{align}\label{eq:td-update}
    V(x) \leftarrow V(x) + \alpha( R + \gamma V(X') - V(x)) \, .
\end{align}
While the mean-return estimator in Expression~\eqref{eq:td-return-estimate} is generally \emph{biased}, since $V(X')$ is not generally equal to the true expected return $V^\pi(X')$, it is often a lower-variance estimate, since we are replacing the \emph{random} return from $X'$ with an estimate of its expectation \citep{sutton1988learning,sutton2018reinforcement,kearns2000bias}.

This motivates the TD learning rule in Expression~\eqref{eq:td-update} based on the Monte Carlo update rule in Expression~\eqref{eq:mc-update}, with the understanding that this algorithm can be applied more generally, with access only to sampled transitions (rather than full trajectories), and may result in more accurate estimates of the value function, due to lower-variance updates, and the propensity of TD algorithms to ``share information'' across states. Note however that this does not \emph{prove} anything about the behaviour of temporal-difference learning, and a fully rigorous theory of the asymptotic behaviour emerged several years after TD methods were formally introduced \citep{sutton1984temporal,sutton1988learning,watkins1989learning,watkins1992q,dayan1992convergence,dayan1994td,jaakkola1994convergence,tsitsiklis1994asynchronous}.

\section{Quantile Temporal-Difference Learning and Quantile Dynamic Programming}
\label{sec:qtd-and-qdp}

We now present the main algorithms of study in this paper.

\subsection{Quantile Regression}\label{sec:quantile-regression}

To motivate QTD, we begin by considering how we might adapt a Monte Carlo algorithm such as that in Expression~\eqref{eq:mc-update} to learn about the distribution of returns, rather than just their expected value.
We cannot learn the return distribution in its entirety with a finite collection of parameters; the space of return distributions is infinite-dimensional, so we must instead be satisfied with learning an approximation of the return distribution by selecting a \emph{probability distribution representation} \citep[Chapter~5]{bdr2022}: a subset of probability distributions parametrised by a finite-dimensional set of parameters.
The approach of quantile temporal-difference learning is to learn an approximation of the form
\begin{align}\label{eq:quantile-rep}
    \eta(x) = \sum_{i=1}^m \frac{1}{m} \delta_{\theta(x, i)} \, ;
\end{align}
an equally-weighted mixture of Dirac deltas, for each state $x \in \mathcal{X}$. The quantile-based approach to distributional reinforcement learning aims to have the particle locations $(\theta(x, i))_{i=1}^m$ approximate certain \emph{quantiles} of $\eta^\pi(x)$.

\begin{definition}\label{def:quantile}
    For a probability distribution $\nu \in \mathscr{P}(\mathbb{R})$ and parameter $\tau \in (0, 1)$, the set of $\tau$-quantiles of $\nu$ is given by the set
    \begin{align*}
        \{ z \in \mathbb{R} : F_\nu(z) = \tau \} \cup \inf\{ y \in \mathbb{R} : F_\nu(y) > \tau \} \, ,
    \end{align*}
    where $F_\nu : \mathbb{R} \rightarrow [0, 1]$ is the CDF of $\nu$, defined by $F_\nu(t) = \mathbb{P}_{Z \sim \nu}(Z \leq t)$ for all $t \in \mathbb{R}$.
\end{definition}

Expanding on this definition, if the set $\{ z \in \mathbb{R} : F_\nu(z) = \tau \}$ is non-empty, then the $\tau$-quantiles are precisely the values $z$ such that $\mathbb{P}_{Z \sim \nu}(Z \leq z) = \tau$. If however this set is empty (which may arise when $F_\nu$ has points of discontinuity), then the quantile is the smallest value $y$ such that $\mathbb{P}_{Z \sim \nu}(Z \leq y) > \tau$.
Note also that if $F_\nu$ is strictly increasing, this guarantees uniqueness of each $\tau$-quantile for $\tau \in (0, 1)$; this is often a useful property in the analysis we consider later. These different cases are illustrated in Figure~\ref{fig:quantiles}. The generalised inverse CDF of $\nu$, $F^{-1}_\nu : (0, 1) \rightarrow \mathbb{R}$, is defined by
\begin{align*}
    F^{-1}_\nu(\tau) = \inf \{ y : F_\nu(y) \geq \tau \} \, ,
\end{align*}
and provides a way of uniquely specifying a quantile for each level $\tau$. In cases where there is not a unique $\tau$-quantile (see Figure~\ref{fig:quantiles}), $F^{-1}_\nu(\tau)$ corresponds to the \emph{left-most} or \emph{least} valid $\tau$-quantile. We also introduce the notation
\begin{align*}
    \bar{F}^{-1}_\nu(\tau) = \inf \{ y : F_\nu(y) > \tau \} \, ,
\end{align*}
which corresponds to the \emph{right-most} or \emph{greatest} $\tau$-quantile; notice the strict inequality that appears in the definition, in contrast to that of $F^{-1}_\nu(\tau)$. If $F^{-1}_\nu$ is continuous at $\tau$, then $F^{-1}_\nu(\tau) = \bar{F}^{-1}_\nu(\tau)$, as is the case for $\tau = \tau_1$ and $\tau = \tau_3$ in Figure~\ref{fig:quantiles}. However, if $F_\nu$ has a flat region for the value $\tau$ (as is the case for $\tau=\tau_2$ in Figure~\ref{fig:quantiles}), then $F^{-1}_\nu(\tau)$ and $\bar{F}^{-1}_\nu(\tau)$ are distinct, and correspond to the boundary points of this flat region.

\begin{figure}
    \centering
    \includegraphics[keepaspectratio,width=.6\textwidth]{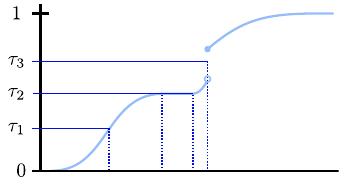}
    \caption{The three distinct scenarios that arise in defining quantiles. Firstly, there is a value $z_1$ for which $F_\nu(z_1)=\tau_1$ and at which $F_\nu$ is strictly increasing. Therefore $z_1$ is the unique $\tau_1$-quantile of $\nu$. Next, there is an interval $[z_2, z_2']$ on which $F_\nu$ equals $\tau_2$, therefore all elements in this interval are $\tau_2$-quantiles of $\nu$. Finally, there is no value $z$ such that $F_\nu(z) = \tau_3$, and the unique $\tau_3$-quantile is therefore defined by the infimum part of the definition.}
    \label{fig:quantiles}
\end{figure}

Algorithmically, we aim for $\theta(x, i)$ to approximate a $\tau_i$-quantile of $\eta^\pi(x)$, where $\tau_i = \nicefrac{2i-1}{2m}$.
To build learning algorithms that achieve this, we require an incremental algorithm that updates $\theta(x, i)$ in response to samples from the target distribution $\eta^\pi(x)$, which converges to a $\nicefrac{2i-1}{2m}$-quantile of $\eta^\pi(x)$.

Such an approach is available by using the quantile regression loss. We define the quantile regression loss associated with distribution $\nu \in \mathscr{P}(\mathbb{R})$ and quantile level $\tau \in (0, 1)$ as a function of $v$ by
\begin{align}\label{eq:qr-loss}
    \mathbb{E}_{Z \sim \nu}[( \tau \mathbbm{1}\{Z \geq v \} + (1-\tau) \mathbbm{1}\{ Z < v \} ) | Z - v |] \, .
\end{align}
This loss is the expectation of an asymmetric absolute value loss, in which positive and negative errors are weighted according to the parameters $\tau $ and $1-\tau$ respectively. Just as the expected squared loss encountered above encodes the mean as its unique minimiser, the quantile regression loss encodes the $\tau$-quantiles of $\nu$ as the unique minimisers; see, for example, \citet{koenker2005quantile} for further background. Thus, applying the quantile regression loss to the problem of estimating $\tau$-quantiles of the return distribution, we arrive at the loss 
\begin{align*}
    \mathcal{L}^{\tau, \pi}_x(v) = \mathbb{E}^\pi_x\Big[( \tau \mathbbm{1}\{\Delta \geq 0 \} + (1-\tau) \mathbbm{1}\{ \Delta < 0 \} ) | \Delta |\Big] \, , \quad \text{where } \Delta =   \sum_{t=0}^\infty \gamma^t R_t - v \, .
\end{align*}
Given an observed return $\sum_{t \geq 0} \gamma^t R_t$ from the state $x$, we therefore have that an unbiased estimator of the negative gradient\footnote{Technically speaking, we are assuming that differentiation and expectation can be interchanged here. Further, under certain circumstances the loss is only \emph{sub}-differentiable. As our principal goal in this section is to provide intuition for QTD, we do not comment further on these technical details here. The convergence results later in the paper deal with these issues carefully.} of this loss is
\begin{align*}
    \tau \mathbbm{1}\Big\{\sum_{t = 0}^\infty \gamma^t R_t \geq v \Big\} - (1-\tau) \mathbbm{1}\Big\{ \sum_{t = 0}^\infty \gamma^t R_t < v \Big\} \, ,
\end{align*}
which motivates an update rule of the form
\begin{align}\label{eq:mc-qr}
    \theta(x, i) \leftarrow \theta(x, i) + \alpha \Big( \tau_i \mathbbm{1}\Big\{\sum_{t = 0}^\infty \gamma^t R_t \geq \theta(x, i) \Big\} - (1-\tau_i) \mathbbm{1}\Big\{ \sum_{t = 0}^\infty \gamma^t R_t < \theta(x, i) \Big\} \Big) \, .
\end{align}
This can be rewritten as
\begin{align}\label{eq:mc-qr-simple}
    \theta(x, i) \leftarrow \theta(x, i) + \alpha \Big( \tau_i - \mathbbm{1}\Big\{\sum_{t = 0}^\infty \gamma^t R_t < \theta(x, i) \Big\} \Big) \, .
\end{align}
This is essentially the application of the stochastic gradient descent method for quantile regression to learning quantiles of the return distribution.

\subsection{Quantile Temporal-Difference Learning}
\label{sec:qtd}\label{sec:qtd-alg}

We can motivate and describe the quantile temporal-difference learning algorithm \citep{dabney2018distributional,bdr2022} by modifying the Monte Carlo algorithm in Expression~\eqref{eq:mc-qr} in a similar manner to the modification that led to the TD algorithm in Expression~\eqref{eq:td-update}. We replace the Monte Carlo return
\begin{align*}
    \sum_{t=0}^\infty \gamma^t R_t
\end{align*}
based on a full trajectory, with an approximate sample from the return distribution derived from an observed transition $(x, R, X')$, and the estimate $\eta(X')$ of the return distribution at state $X'$. If the return distribution estimate $\eta(X')$ takes the form
given in Equation~\eqref{eq:quantile-rep}, as is the case for the probability distribution representation considered here, then such a sample return is obtained as
\begin{align*}
    R + \gamma \theta(X', J) \, ,
\end{align*}
with $J$ sampled uniformly from $\{1,\ldots,m\}$. This yields the update rule
\begin{align*}
    \theta(x, i) \leftarrow \theta(x, i) + \alpha \Big( \tau_i - \mathbbm{1}\Big\{R + \gamma \theta(X', J) < \theta(x, i) \Big\} \Big) \, .
\end{align*}
We can consider also a variance-reduced version of this update, in which we average over updates performed under different realisations of $J$, leading to the update
\begin{align}\label{eq:qtd-update}
    \theta(x, i) \leftarrow \theta(x, i) + \frac{\alpha}{m} \sum_{j=1}^m \Big( \tau_i - \mathbbm{1}\Big\{R + \gamma \theta(X', j) < \theta(x, i) \Big\} \Big) \, .
\end{align}
This is precisely the quantile temporal-difference learning update, presented in Algorithm~\ref{alg:qtd} below, which underlies many recent successful applications of reinforcement learning at scale \citep{dabney2018distributional,dabney2018implicit,yang2019fully,bellemare2020autonomous,wurman2022outracing,fawzi2022discovering}. Similar to other temporal-difference learning algorithms, QTD updates its parameters $((\theta(x, i))_{i=1}^m : x \in \mathcal{X})$ on the basis of sample transitions $(x, r, x')$ generated through interaction with the environment via the policy $\pi$, comprising a state, reward, and next state.

\begin{algorithm}
    \begin{algorithmic}[1]
        \REQUIRE Quantile estimates $\theta \in \mathbb{R}^{\mathcal{X} \times [m]}$ , \\
        $\quad\quad\ \ \,$ Observed transition $(x, r, x')$ , \\
        $\quad\quad\ \ \,$ Learning rate $\alpha$.
        \STATE Set $\tau_i = \tfrac{2i-1}{2m}$ for each $i=1,\ldots,m$.
        \FOR{$i = 1,\ldots,m$}
            \STATE Set $\theta'(x, i) \leftarrow \theta(x, i) + \alpha \frac{1}{m} \sum_{j=1}^m \left\lbrack \tau_i - \mathbbm{1}\Big\{r + \gamma \theta(x', j) - \theta(x, i) < 0 \Big\} \right\rbrack $
        \ENDFOR
        \FOR{$i = 1,\ldots,m$}
            \STATE Set $\theta(x, i) \leftarrow \theta'(x, i)$
        \ENDFOR
        \STATE \textbf{return} $((\theta'(x, i))_{i=1}^m : x \in \mathcal{X})$
    \end{algorithmic}
    \caption{QTD update}
    \label{alg:qtd}
\end{algorithm}

Whilst the QTD update makes use of temporal-difference errors $r + \gamma \theta(x', j) - \theta(x, i)$, there are two key differences to the use of analogous quantities in classical TD learning. First, the TD errors influence the update only through their sign, not their magnitude. Second, the predictions at each state $(\theta(x, i))_{i=1}^m$ are indexed by $i$, and each update includes a distinct term $\tau_i$ (equal to $\nicefrac{2i-1}{2m}$). The presence of these terms causes the learnt parameters to make distinct predictions, as described in Section~\ref{sec:quantile-regression}. Practical implementations of QTD use these precise values for $\tau_i$, equally spaced out on $[0,1]$, as proposed by \citet{dabney2018distributional}.
Much of the analysis in this paper goes through straightforwardly for other values
of $\tau_i$, though we will see in Section~\ref{sec:fixed-point-analysis} that this choice is well motivated in that it provides the best bounds on distribution approximation.
The tabular QTD algorithm as described in Algorithm~\ref{alg:qtd} uses a factor $O(m)$ times more memory than an analogous classical TD algorithm, owing to the need to store multiple predictions at each state, though the scaling with the size of the state space is the same as for classical TD. For further discussion of the computational complexity of QTD, see \citet[Appendix~A.3;][]{rowland2023statistical}.
Much of the analysis in this paper goes through straightforwardly for other values of $\tau_i$, though we will see in Section~\ref{sec:fixed-point-analysis} that this choice is well motivated in that it provides the best bounds on distribution approximation.

The discussion above provides \emph{motivation} for the form of the QTD update given in Algorithm~\ref{alg:qtd}, and intuition as to why this algorithm might perform reasonably, and learn a sensible approximation to the return distribution. However, it stops short of providing an explanation of how the algorithm should be expected to behave, or providing any theoretical guarantees as to what the algorithm will in fact converge to. A core goal of the sections that follow is to answer these questions, and put QTD on firm theoretical footing.

\subsection{Motivating Examples}

Before undertaking an analysis of QTD, we pause to provide several numerical examples of its behaviour in small environments.
These examples provide further intuition for the characteristics of the algorithm, illustrate the breadth of qualitative behaviours it can exhibit, and provide motivation for the kinds of theoretical questions we might hope to answer.

\begin{example}
    Consider the chain MDP illustrated at the top of Figure~\ref{fig:example1}. The random return at each state is a sum of independent Gaussian random variables, and hence the return distribution at each state is Gaussian.
    The centre plot in Figure~\ref{fig:example1} illustrates the evolution of $m=5$ quantile estimates learnt by QTD, using a constant learning rate of 0.01, and updating all states at each update. The estimated quantile values eventually settle after around 6,000 updates, with small oscillations around this point.
    The bottom of Figure~\ref{fig:example1} compares the true return distribution at each state (in blue), with the approximation learnt by QTD (in black), and the approximation obtained with the true value of the five quantiles of interest (grey).
    The behaviour of QTD in this case raises several questions:
    Can it be shown that QTD is guaranteed to stabilise/converge around a certain point?
    Can a guarantee be given on the quality of the approximate distributions learnt by QTD?
\end{example}

\begin{figure}
    \centering
    \includegraphics[keepaspectratio,width=.75\textwidth]{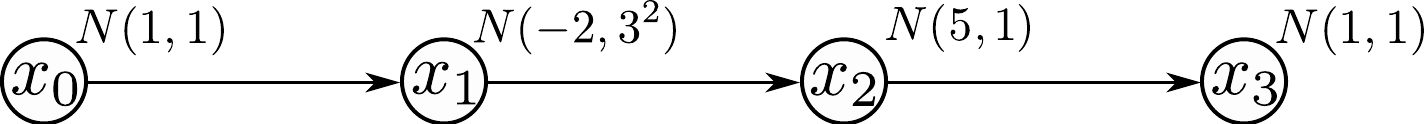}
    
    \includegraphics[keepaspectratio,width=.8\textwidth]{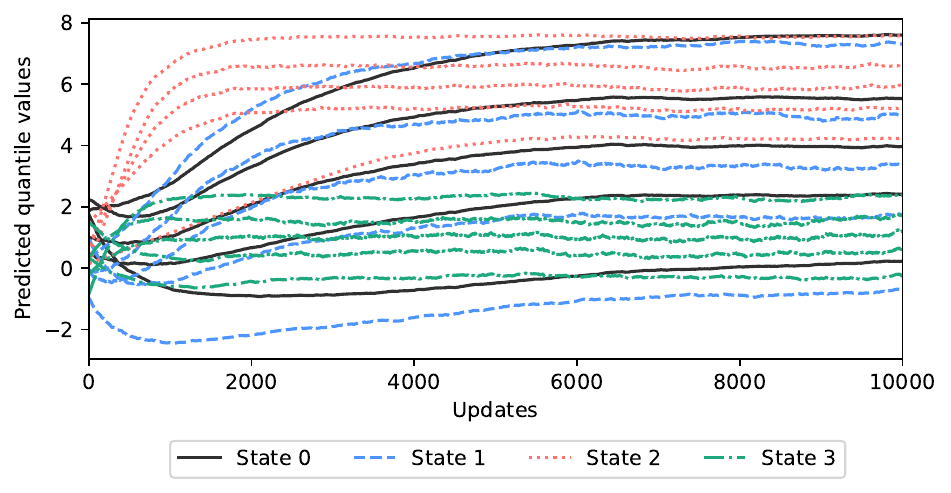}
    
    \null\hfill
    \includegraphics[keepaspectratio,width=.20\textwidth]{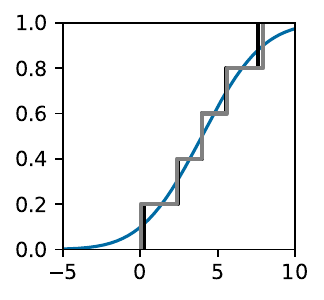}
    \hfill
    \includegraphics[keepaspectratio,width=.20\textwidth]{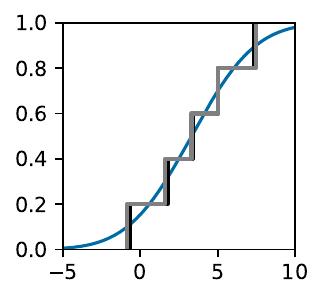}
    \hfill
    \includegraphics[keepaspectratio,width=.20\textwidth]{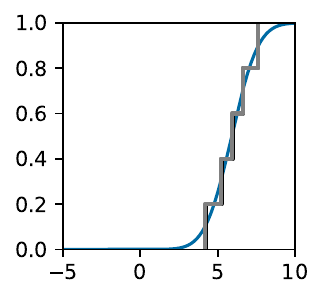}
    \hfill
    \includegraphics[keepaspectratio,width=.20\textwidth]{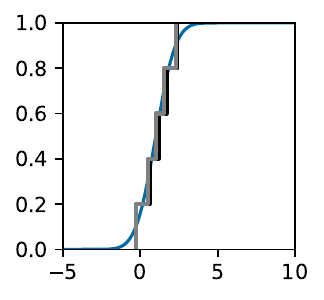}
    \hfill\null
    
    \null\hfill
    \includegraphics[keepaspectratio,width=.20\textwidth]{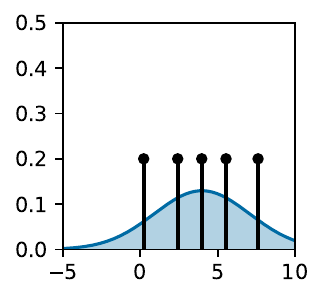}
    \hfill
    \includegraphics[keepaspectratio,width=.20\textwidth]{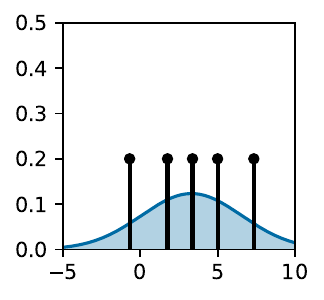}
    \hfill
    \includegraphics[keepaspectratio,width=.20\textwidth]{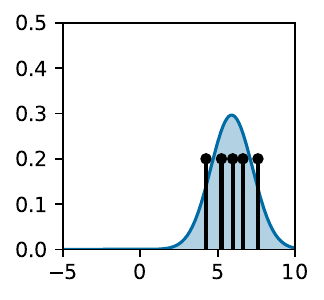}
    \hfill
    \includegraphics[keepaspectratio,width=.20\textwidth]{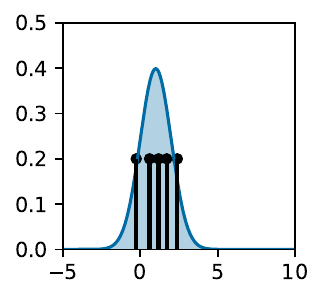}
    \hfill\null
    
    \caption{
    Top: A chain MDP with four states. Each transition yields a normally-distributed reward; from $x_3$, the episode ends. The discount factor is $\gamma = 0.9$.
    Centre-top: The progress of QTD, run with $m=5$ quantiles, over the course of 10,000 updates. The vertical axis corresponds to the predicted quantile values.
    Centre-bottom: The true CDF of the return distribution (blue) at each state, along with the final estimate produced by QTD (black), and the approximation produced by the quantiles of the return distribution (grey).
    Bottom: The PDF of the return distribution (blue) at each state, along with the final quantile approximation produced by QTD (black).
    }
    \label{fig:example1}
\end{figure}

\begin{example}\label{ex:2d}
    For a different perspective on the behaviour of QTD, consider a two-state MDP with transition dynamics as illustrated in the top-left of Figure~\ref{fig:example2}, and discount factor $\gamma = 0.5$. The reward obtained when transitioning from state $x_1$ is distributed as N(2, 1), and the reward obtained when transitioning from state $x_2$ is distributed as N(-1, 1); here, we write $N(\mu, \sigma^2)$ for the normal distribution with mean $\mu$ and variance $\sigma^2$. We consider the case of learning a single quantile (the median) at each of these two states, taking $m=1$; this allows us to plot the full phase space of the QTD algorithm in a two-dimensional plot.
    
    The top-right of Figure~\ref{fig:example2} shows a path taken by QTD under this MDP. In addition, the streamplot illustrates the direction of the \emph{expected} update that QTD undertakes at each point in phase space. We empirically observe convergence of the algorithm to a point. Additionally, the expected update direction changes smoothly; the result is a vector field that appears to point towards the point of convergence from all directions.
    
    The bottom-left of Figure~\ref{fig:example2} shows a path taken by QTD under a modified version of the MDP, in which the reward distributions N(2, 1) and N(-1, 1) are replaced with $\delta_2$ and $\delta_{-1}$, respectively. We observe that the algorithm still converges to a point, although the vector field of expected update directions is now piecewise constant, with discontinuities along several lines. This behaviour is typical of QTD; the less `smooth' the reward distributions in the MDP, the more abrupt the changes in behaviour we typically observe with QTD.
    
    Finally, we consider a modified version of the MDP in which all transition probabilities are $\nicefrac{1}{2}$, rewards from state $x_1$ are always $2$, and rewards from state $x_2$ are always $-1$. In this case, QTD no longer appears to converge to a point, but instead converges to the set bounded by the four grey lines appearing in the bottom-right of Figure~\ref{fig:example2}, and subsequently performing a random walk over this set. This collection of examples illustrates that QTD can exhibit a fairly wide family of behaviours depending on the characteristics of the environment. In particular, non-uniqueness of quantiles in reward distributions (corresponding to flat regions in reward distribution CDFs) can lead to multiple possible limit points, and discontinuities in reward distributions can lead to discontinuous changes in expected updates; by contrast, reward distributions that are absolutely continuous lead to smooth changes in expected dynamics.
\end{example}

\begin{figure}
    \centering
    
    \null\hfill
    \includegraphics[keepaspectratio,width=.38\textwidth,raise=1.7cm]{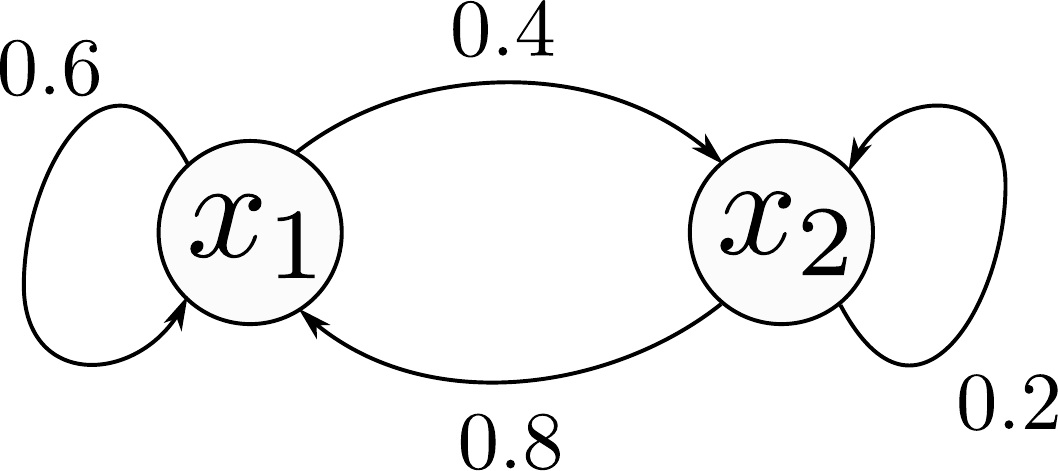}
    \hfill
    \includegraphics[keepaspectratio,width=.38\textwidth]{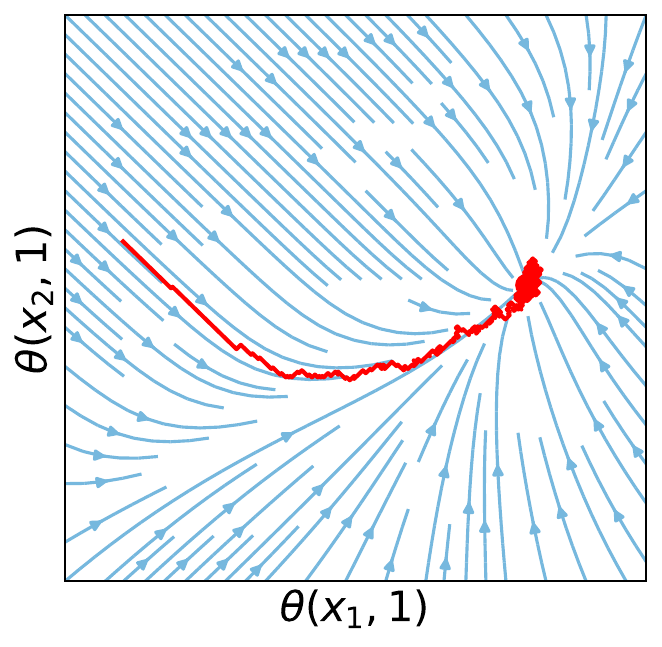}
    \hfill\null
    
    \null\hfill
    \includegraphics[keepaspectratio,width=.38\textwidth]{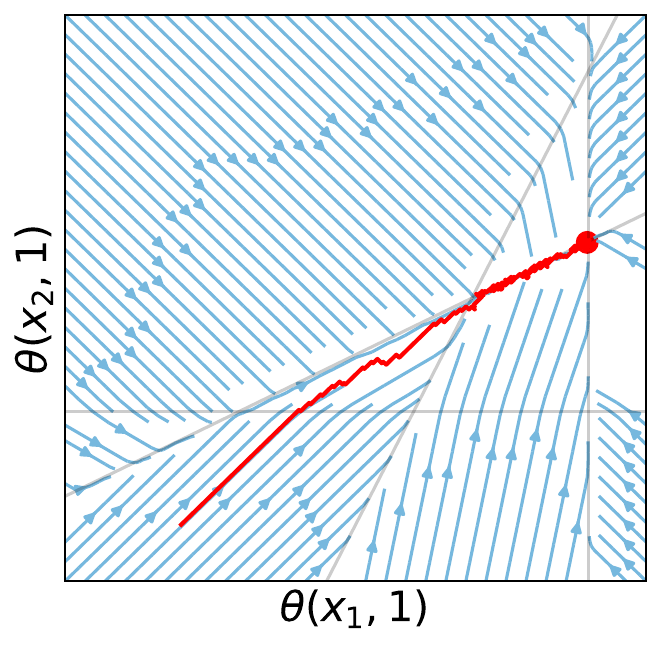}
    \hfill
    \includegraphics[keepaspectratio,width=.38\textwidth]{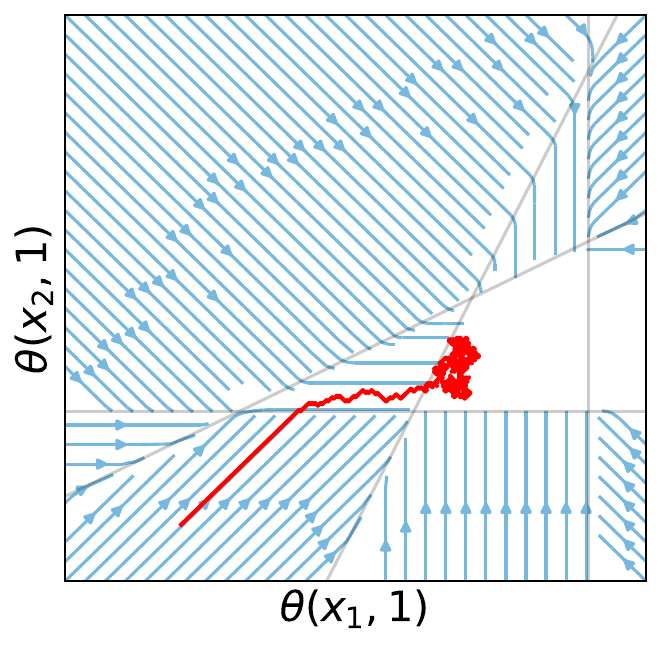}
    \hfill\null
    
    \caption{
    Top left: The example Markov decision process described in Example~\ref{ex:2d}.
    Top right: Example dynamics of QTD with $m=1$ in this environment, when reward distributions are Gaussian. Also included are the directions of expected update, in blue.
    Bottom left: Example dynamics and expected update directions when reward distributions are Dirac deltas.
    Bottom right: Example dynamics and expected updates with modified environment transition probabilities.}
    \label{fig:example2}
\end{figure}

\subsection{Quantile Dynamic Programming}\label{sec:dp}

Recall the QTD update given in Equation~\eqref{eq:qtd-update}. As described in Section~\ref{sec:qtd-alg}, this update serves, on average, to move $\theta(x, i)$ in the direction of the $\tau_i$-quantiles of the distribution of the random variable $R + \theta(X', J)$, where $(x, R, X')$ is a random transition generated by interacting with the environment using $\pi$, and $J \sim \text{Unif}(\{1,\ldots,m\})$.

Suppose we were able to update $\theta(x, i)$ not just with a single gradient step in this direction, but instead were able to update it to take on exactly this quantile value. This motivates a dynamic programming alternative to QTD, \emph{quantile dynamic programming} (QDP), which directly calculates these quantiles iteratively, in a similar manner to iterative policy evaluation in classical reinforcement learning \citep{bertsekas1996neuro}.

The mathematical structure of such an algorithm is given in Algorithm~\ref{alg:qdp}. This stops short of being an implementable algorithm, since we do not describe in what format the transition probabilities and reward distributions are available, which are required to evaluate the inverse CDFs that arise in the algorithm. 
However, for MDPs in which transition probabilities and reward distributions are available, QDP is an algorithmic framework of interest in its own right, and to this end we provide several concrete implementations in Appendix~\ref{sec:qdp-implementations}.

The QDP template in Algorithm~\ref{alg:qdp} is parametrised by the interpolation parameters $\lambda \in [0,1]^{\mathcal{X} \times [m]}$. These parameters control exactly which quantile is chosen when the desired quantile level $\tau_i$ corresponds to a flat region of the CDF for the distribution $\nu$ (the second case in Figure~\ref{fig:quantiles}). QDP was originally presented by \citet{bdr2022} in the case $\lambda(x, i) \equiv 0$; the presentation here generalises QDP to a family of algorithms, parametrised by $\lambda$.

Our interest in QDP stems from the fact that QTD can be viewed as approximating the behaviour of the QDP algorithms, without requiring access to the transition structure and reward distributions of the environment. In particular, we will show that under appropriate conditions, the asymptotic behaviour of QTD and QDP are equivalent: they both converge to the same limiting points. Figure~\ref{fig:qdp-ex} illustrates the behaviour of the QDP algorithm in the environment described in Example~\ref{ex:2d}; since the reward distributions in this example have strictly increasing CDFs, QDP behaves identically for all choices of interpolation parameters $\lambda$. QTD and QDP appear to have the same asymptotic behaviour, converging to the same limiting point. In cases where QTD appears to converge to a set, such as in the bottom-right plot of Figure~\ref{ex:2d}, the relationship is slightly more complicated, and there is a correspondence between the asymptotic behaviour of QTD and the family of dynamic programming algorithms parametrised by $\lambda$, as illustrated at the bottom of Figure~\ref{fig:qdp-ex}. Thus, to understand the asymptotic behaviour of QTD, we begin by analysing the asymptotic behaviour of QDP.

\begin{figure}
    \centering
    
    \null\hfill
    \includegraphics[keepaspectratio,width=.4\textwidth]{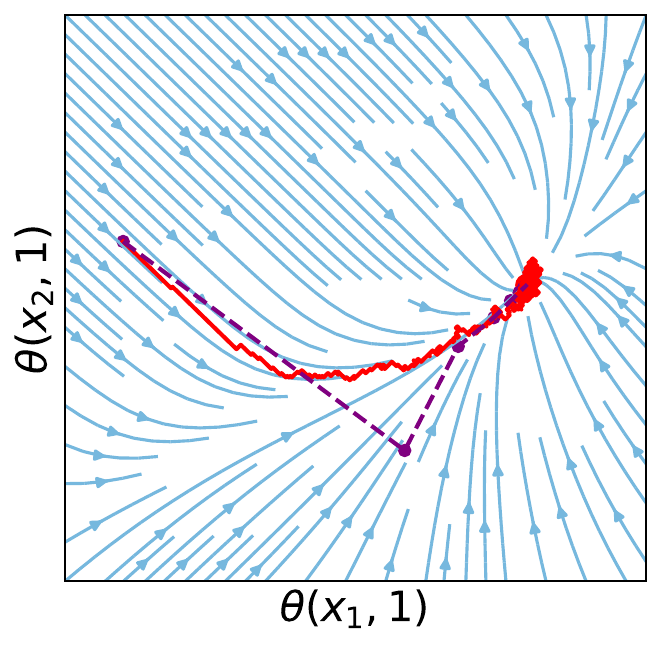}\hfill
    \includegraphics[keepaspectratio,width=.4\textwidth]{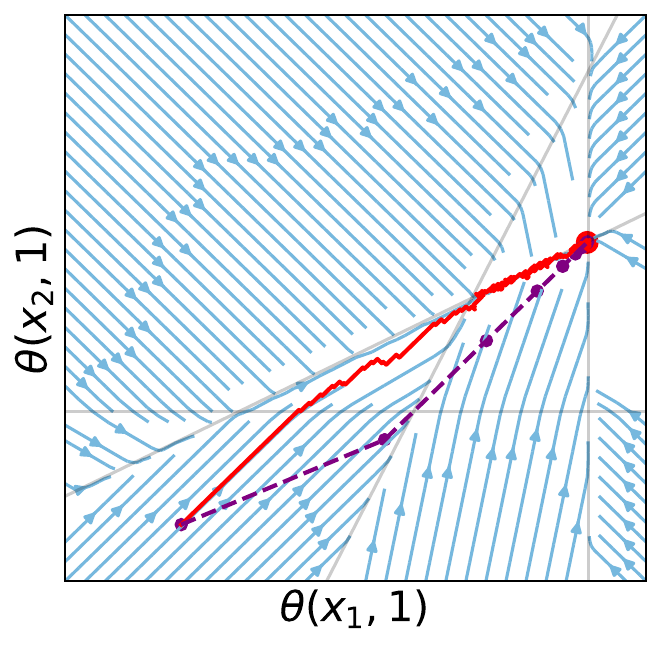}
    \hfill\null
    
    \null\hfill
    \includegraphics[keepaspectratio,width=.4\textwidth]{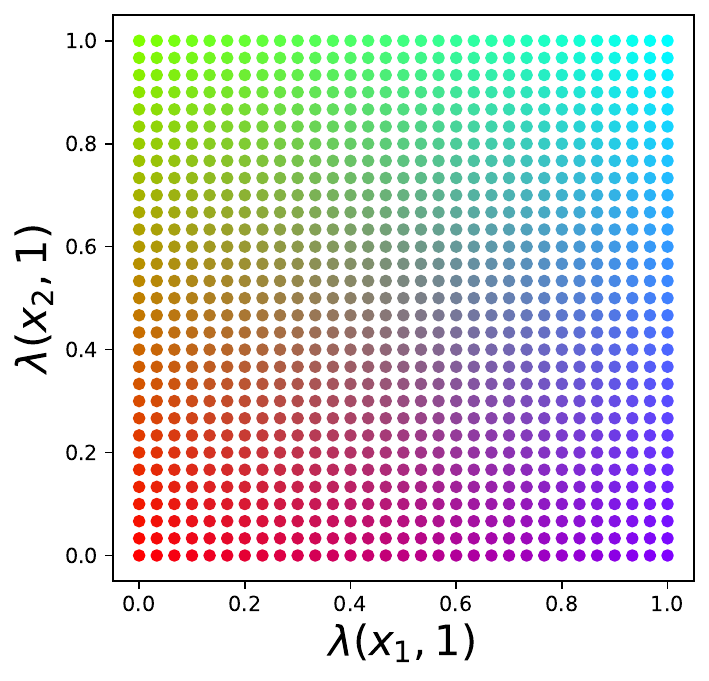}\hfill
    \includegraphics[keepaspectratio,width=.4\textwidth]{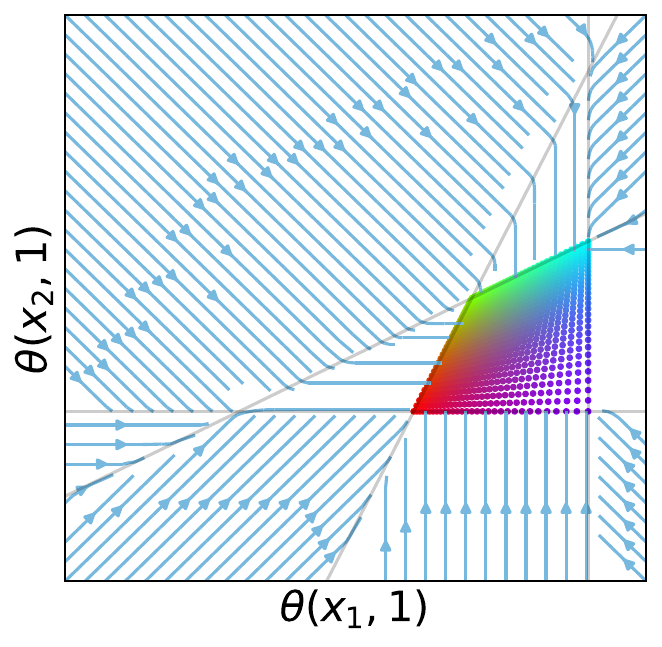}
    \hfill\null
    
    \caption{
    Top left: Illustration of QDP (dashed purple) and QTD (solid red) on the first MDP from Example~\ref{ex:2d}, with Gaussian rewards.
    Top right: Illustration of QDP and QTD on the second MDP from Example~\ref{ex:2d}, with deterministic rewards.
    Bottom: Values of $\lambda$ and corresponding fixed points of QDP in the final MDP from Example~\ref{ex:2d}.}
    \label{fig:qdp-ex}
\end{figure}

\begin{algorithm}[h]
    \begin{algorithmic}[1]
        \REQUIRE Quantile estimates $((\theta(x, i))_{i=1}^m : x \in \mathcal{X})$, \\
        Interpolation parameters $\lambda \in [0, 1]^{\mathcal{X} \times [m]}$.
        \FOR{$x \in \mathcal{X}$}
            \STATE Let $(x, R, X')$ be a random transition under $\pi$, and $J \sim \text{Unif}(\{1,\ldots,m\})$.
            \STATE Set $\nu$ to be the distribution of $R + \gamma \theta(X', J)$.
            \FOR{$i =1,\ldots,m$}
                \STATE Set $\theta(x, i) \leftarrow (1 - \lambda(x, i)) F^{-1}_\nu(\tau_i) + \lambda(x, i) \bar{F}^{-1}_\nu(\tau_i)$. \label{algline:quantiles}
            \ENDFOR
        \ENDFOR
        \STATE \textbf{return} $((\theta(x, i))_{i=1}^m : x \in \mathcal{X})$
    \end{algorithmic}
    \caption{Quantile dynamic programming}
    \label{alg:qdp}
\end{algorithm}

\section{Convergence of Quantile Dynamic Programming}\label{sec:dp-analysis}

We can decompose the update QDP performs into the composition of several operators. Algorithm~\ref{alg:qdp} manipulates tables of the form $((\theta(x, i))_{i=1}^m : x \in \mathcal{X})$.
For a given state $x$, the vector $(\theta(x, i))_{i=1}^m$ represents the estimated $\nicefrac{2i-1}{2m}$-quantiles of the return distribution at state $x$, for $i=1,\ldots,m$. In mathematically analysing the algorithm, it is useful to be able to refer to the distribution encoded by these quantiles:
\begin{align}\label{eq:quantiles-to-dist}
    \frac{1}{m}\sum_{i=1}^m \delta_{\theta(x, i)} \, ,
\end{align}
and reason about the transformations undertaken by Algorithm~\ref{alg:qdp} directly in terms of distributions. To this end, if we write $\eta(x) \in \mathscr{P}(\mathbb{R})$ for the probability distribution associated with the quantile estimates $(\theta(x, i))_{i=1}^m$, we can interpret the transformation performed by Algorithm~\ref{alg:qdp} as comprising two parts, which we now describe in turn.

First, the variable $\eta(x)$ is assigned the distribution of $R + \gamma G(X')$, where $R, X'$ are the random reward and next-state encountered from the initial state $x$ with policy $\pi$, and $(G(y) : y \in \mathcal{X})$ is an independent collection of random variables, with each $G(y)$ distributed according to $\eta(y)$.

We write $\mathcal{T}^\pi : \mathscr{P}(\mathbb{R})^\mathcal{X} \rightarrow \mathscr{P}(\mathbb{R})^\mathcal{X}$ for this transformation. The function $\mathcal{T}^\pi$ is known as the \emph{distributional Bellman operator} \citep{bellemare2017distributional,rowland2018analysis,bdr2022}. In terms of the above definition via distributions of random variables, $\mathcal{T}^\pi$ can be written
\begin{align*}
    (\mathcal{T}^\pi \eta)(x) = \mathcal{D}_\pi( R + \gamma G(X') ) \, ,
\end{align*}
where $(x, R, X')$ is a random environment transition beginning at $x$, independent of $(G(y) : y \in \mathcal{X})$, and $\mathcal{D}_\pi$ extracts the distribution of its argument when $(x, R, X')$ is generated by sampling an action from $\pi$. See \citet{bdr2022} for further background on the distributional Bellman operator.

In general, $\mathcal{T}^\pi \eta$ may comprise much more complicated distributions than $\eta$ itself, with many more atoms, or possibly infinite support, if reward distributions are infinitely-supported. Algorithm~\ref{alg:qdp} does not return these full transformed distributions, but rather approximations, or \emph{projections}, of these distributions, obtained by keeping only information about certain quantiles (in the inner for-loop of Algorithm~\ref{alg:qdp});
this is the second distribution transformation the algorithm undertakes. Each choice of interpolation parameters $\lambda$ corresponds to a different projection operator, denoted $\Pi^\lambda : \mathscr{P}(\mathbb{R})^\mathcal{X} \rightarrow \mathscr{P}(\mathbb{R})^\mathcal{X}$, and defined by
\begin{align}\label{eq:proj-definition}
    (\Pi^\lambda \eta )(x) = \frac{1}{m} \sum_{i=1}^m \delta_{(1-\lambda(x, i)) F^{-1}_{\eta(x)}(\tau_i) + \lambda(x, i) \bar{F}^{-1}_{\eta(x)}(\tau_i)} \, .
\end{align}
Thus, the composition $\Pi^\lambda \mathcal{T}^\pi$, the projected distributional Bellman operator, is a transformation on the space of return-distribution functions $\mathscr{P}(\mathbb{R})^\mathcal{X}$. 
We will also find it useful to abuse notation slightly and consider $\Pi^\lambda \mathcal{T}^\pi$ as an operator on the space $\mathbb{R}^{\mathcal{X} \times [m]}$ of parameters that QDP and QTD operate over. The understanding is that an input $\theta \in \mathbb{R}^{\mathcal{X} \times [m]}$ is first re-interpreted as a collection of distributions as in Expression~\eqref{eq:quantiles-to-dist},
with $\Pi^\lambda \mathcal{T}^\pi$ applied as defined above to this collection of probability distributions, and then finally extracting the support of the output distributions, which take the form 
\begin{align*}
    \sum_{i=1}^m \frac{1}{m} \delta_{z_i} \, ,
\end{align*}
to return an element of $\mathbb{R}^{\mathcal{X} \times [m]}$. We will also write $\mathcal{T}^\pi \theta$ for the element of $\mathscr{P}(\mathbb{R})^\mathcal{X}$ obtained by applying $\mathcal{T}^\pi$ to the distributions $(\eta(x) : x \in \mathcal{X})$ defined by
\begin{align*}
    \eta(x) = \sum_{i=1}^m \frac{1}{m} \delta_{\theta(x, i)} \, .
\end{align*}

\begin{remark}
    This convention highlights that there are two complementary views of distributional reinforcement learning algorithms, through finite-dimensional sets of parameters, and through probability distributions. The view in terms of probability distributions is often useful in contraction analysis, and in measuring approximation error, while we will see that the parameter view is key to the stochastic approximation analysis that follows, and is ultimately the way in which these algorithms are implemented.
\end{remark}

With this convention, $\Pi^\lambda \mathcal{T}^\pi \theta$ is precisely the table $\theta'$ output by Algorithm~\ref{alg:qdp} on input $\theta$, and so the QDP algorithm is mathematically equivalent to repeated application of the operator $\Pi^\lambda \mathcal{T}^\pi$ to an initial collection of quantile estimates. To understand the long-term behaviour of QDP, we can therefore seek to understand this projected operator $\Pi^\lambda \mathcal{T}^\pi$.

\subsection{Convergence Analysis}

We will show that $\Pi^\lambda \mathcal{T}^\pi$ is a contraction mapping with respect to an appropriate metric over return-distribution functions. Building on the analysis in the case of $\lambda \equiv 0$ carried out by \citet{dabney2018distributional} and \citet{bdr2022}, we use the Wasserstein-$\infty$ metric $w_\infty : \mathscr{P}(\mathbb{R}) \times \mathscr{P}(\mathbb{R}) \rightarrow [0, \infty]$, defined by
\begin{align*}
    w_\infty(\nu, \nu') = \sup_{t \in (0, 1)} |F^{-1}_{\nu}(t) - F^{-1}_{\nu'}(t)| \, ,
\end{align*}
and its extension to return-distribution functions, $\bar{w}_\infty : \mathscr{P}(\mathbb{R})^\mathcal{X} \times \mathscr{P}(\mathbb{R})^\mathcal{X} \rightarrow [0, \infty]$, given by
\begin{align*}
    \bar{w}_\infty(\eta, \eta') = \max_{x \in \mathcal{X}} \sup_{t \in (0, 1)} |F^{-1}_{\eta(x)}(t) - F^{-1}_{\eta'(x)}(t)| \, .
\end{align*}
Both $w_\infty$ and $\bar{w}_\infty$ fulfil all the requirements of a metric, except that they may assign infinite distances (\citeauthor{villani2009optimal}, \citeyear{villani2009optimal}; see also \citeauthor{bdr2022}, \citeyear{bdr2022} for a detailed discussion specifically in the context of reinforcement learning). We must therefore take some care as to when distances are finite.
The following is established by \citet[Proposition~4.15]{bdr2022}.

\begin{proposition}\label{prop:dist-contract}
    The distributional Bellman operator $\mathcal{T}^\pi : \mathscr{P}(\mathbb{R})^{\mathcal{X}} \rightarrow \mathscr{P}(\mathbb{R})^{\mathcal{X}}$ is a $\gamma$-contraction with respect to $\bar{w}_\infty$. That is,
    \begin{align*}
        \bar{w}_\infty(\mathcal{T}^\pi \eta, \mathcal{T}^\pi \eta') \leq \gamma \bar{w}_\infty(\eta, \eta') \, ,
    \end{align*}
    for all $\eta, \eta' \in \mathscr{P}(\mathbb{R})^\mathcal{X}$.
\end{proposition}

Next, we show that the projection operator $\Pi^\lambda$ cannot expand distances as measured by $\bar{w}_\infty$, generalising the proof given by \citet{bdr2022} in the case $\lambda \equiv 0$; the proof is given in Appendix~\ref{sec:proof-proj-nonexpansion}.

\begin{restatable}{proposition}{propProjNonexpansion}\label{prop:proj-nonexpansion}
    The projection operator 
    $\Pi^{\lambda} : \mathscr{P}(\mathbb{R})^\mathcal{X} \rightarrow \mathscr{P}(\mathbb{R})^\mathcal{X}$ is a non-expansion with respect to $\bar{w}_\infty$. That is, for any $\eta, \eta' \in \mathscr{P}(\mathbb{R})^\mathcal{X}$, we have
    \begin{align*}
        \bar{w}_\infty(\Pi^\lambda \eta, \Pi^\lambda \eta') \leq \bar{w}_\infty(\eta, \eta') \, .
    \end{align*}
\end{restatable}

Finally, we put these two results together to obtain our desired conclusion. In stating this result, it is useful here to introduce the notation
\begin{align*}
    \mathscr{F}_{\text{Q}, m} = \Big\{ \sum_{i=1}^m \frac{1}{m} \delta_{z_i} :  z_i \in \mathbb{R} \text{ for } i=1,\ldots,m \Big\} \, ,
\end{align*}
for the set of probability distributions representable with $m$ quantile locations.

\begin{proposition}\label{prop:contraction}
    The projected operator $\Pi^{\lambda} \mathcal{T}^\pi : \mathscr{F}_{\text{Q}, m}^\mathcal{X} \rightarrow \mathscr{F}_{\text{Q}, m}^\mathcal{X}$ is a $\gamma$-contraction with respect to $\bar{w}_\infty$. Hence, $\Pi^{\lambda} \mathcal{T}^\pi$ has a unique fixed point in $\mathscr{F}_{\text{Q}, m}^\mathcal{X}$, which we denote $\qdpfixedpointlambda$. Further, given any initial $\eta_0 \in \mathscr{F}_{\text{Q}, m}^\mathcal{X}$, the sequence $(\eta_k)_{k=0}^\infty$ defined iteratively by $\eta_{k+1} = \Pi^\lambda \mathcal{T}^\pi \eta_k$ for $k \geq 0$ satisfies $\bar{w}_\infty(\eta_k, \qdpfixedpointlambda) \leq \gamma^k \bar{w}_\infty(\eta_0, \qdpfixedpointlambda) \rightarrow 0$.
\end{proposition}
\begin{proof}
    That $\Pi^{\lambda} \mathcal{T}^\pi : \mathscr{F}_{\text{Q}, m}^\mathcal{X} \rightarrow \mathscr{F}_{\text{Q}, m}^\mathcal{X}$ is a $\gamma$-contraction with respect to $\bar{w}_\infty$ follows directly from Propositions~\ref{prop:dist-contract} and \ref{prop:proj-nonexpansion}:
    \begin{align*}
        \bar{w}_\infty(\Pi^\lambda \mathcal{T}^\pi \eta, \Pi^\lambda \mathcal{T}^\pi \eta') \leq \bar{w}_\infty( \mathcal{T}^\pi \eta, \mathcal{T}^\pi \eta') \leq \gamma \bar{w}_\infty(\eta, \eta') \, .
    \end{align*}
    Next, observe that $\bar{w}_\infty$ assigns finite distance to all pairs of return-distribution functions in $\mathscr{F}_{\text{Q}, m}^\mathcal{X}$, and further, this set is complete with respect to $\bar{w}_\infty$. Hence, we may apply Banach's fixed point theorem to obtain the existence of the unique fixed point $\qdpfixedpointlambda$ in $\mathscr{F}_{\text{Q}, m}^\mathcal{X}$. The final claim follows by induction, and the contraction property established for $\Pi^{\lambda} \mathcal{T}^\pi$.
\end{proof}

Note that the fixed point $\qdpfixedpointlambda$ depends on $\lambda$, and therefore implicitly on $m$. We also introduce the notation $\qdpfixedpointlambdaparams \in \mathbb{R}^{\mathcal{X} \times [m]}$ for the parameters of this collection of distributions, which is what the QDP algorithm really operates over, so that we have
\begin{align*}
    \qdpfixedpointlambda(x) = \sum_{i=1}^m \frac{1}{m} \delta_{\qdpfixedpointlambdaparams(x, i)} \, .
\end{align*}
Note that the convergence result of Proposition~\ref{prop:contraction} also implies convergence of the estimated quantile locations to $\hat{\theta}^\pi_\lambda$.
In Section~\ref{sec:fixed-point-analysis}, we will analyse the fixed point $\qdpfixedpointlambda$, and understand how closely it approximates the true return-distribution function $\eta^\pi$. For now, having established convergence of QDP through contraction mapping theory, we can return to QTD and demonstrate its own convergence to the same fixed points.

\section{Convergence of Quantile Temporal-Difference Learning}\label{sec:qtd-convergence}

We now present the convergence analysis of QTD. We will consider a \emph{synchronous} version of QTD, in which all states are updated using independent transitions at each algorithm step, given by:
\begin{align}\label{eq:qtd-sa}
     \theta_{k+1}(x, i) = \theta_k(x, i) +
     \alpha_k \frac{1}{m} \sum_{j=1}^m (\tau_i - \mathbbm{1}\{ R_k(x) + \gamma \theta_k(X'_k(x), j) < \theta_k(x, i)  \}) \, ,
\end{align}
where given $x$ and $k$, we have $(R_k(x), X'_k(x)) \sim P^\pi(\cdot|x)$, independently of the transitions used at all other states/time steps, and $(\alpha_k)_{k=0}^\infty$ is a sequence of step sizes. The assumption of synchronous updates makes the analysis easier to present, and means that our results follow classical approaches to stochastic approximation with differential inclusions \citep{benaim2005stochastic}. It is also possible to extend the analysis to the asynchronous case, where a single state is updated at each algorithm time step (as would be the case in fully online QTD, or an implementation using a replay buffer); see Section~\ref{sec:main-async}. We now state the main convergence result of the paper.
\begin{theorem}\label{thm:general-convergence}
    Consider the sequence $(\theta_k)_{k=0}^\infty$ defined by an initial point $\theta_0 \in \mathbb{R}^{\mathcal{X} \times [m]}$, the iterative update in Equation~\eqref{eq:qtd-sa}, and non-negative step sizes satisfying the condition
    \begin{align}\label{eq:rm}
        \sum_{t=0}^\infty \alpha_k = \infty \, , \quad \alpha_k = o(1/\log k) \, .
    \end{align}
    Then $(\theta_k)_{k=0}^\infty$ converges almost surely to the set of fixed points of the projected distributional Bellman operators $\{ \Pi^\lambda \mathcal{T}^\pi : \lambda \in [0, 1]^{\mathcal{X} \times [m]}\}$; that is,
    \begin{align*}
        \inf_{\lambda \in [0,1]^{\mathcal{X} \times [m]}}\| \theta_k - \qdpfixedpointlambdaparams \|_\infty \rightarrow 0
    \end{align*}
    with probability 1.
\end{theorem}

Of particular note is the generality of this result. It does not require finite-variance conditions on rewards (as is typically the case with convergence results for classical TD); it holds for any collection of reward distributions with the finite mean property set out at the beginning of the paper. Some intuition as to why this is the case is that the finite-variance conditions typically encountered are to ensure that the updates performed in classical TD learning cannot grow in magnitude too rapidly. Since the updates performed in QTD are bounded, this is not a concern, meaning that the proof does not rely on such conditions. We note also that the step size conditions are weaker than the typical Robbins-Monro conditions used in classical TD analyses (see, for example, \citeauthor{bertsekas1996neuro}, \citeyear{bertsekas1996neuro}), which enforce square-summability, also to avoid the possibility of divergence due to unbounded noise in the classical TD learning.

The proof is based on the ODE method for stochastic approximation; in particular we use the framework set out by \citet{benaim1999dynamics} and \citet{benaim2005stochastic}. This involves interpreting the QTD update as a noisy Euler discretisation of a differential equation (or more generally, a differential inclusion). The broad steps are then to argue that the trajectories of the differential equation/inclusion converge to some set of fixed points in a suitable way (that is, in such a way that is robust to small perturbations), and that the asymptotic behaviour of QTD, forming a noisy Euler discretisation, matches the asymptotic behaviour of the true trajectories. This then allows us to deduce that the QTD iterates converge to the same set of fixed points as the true trajectories. We begin by elucidating the connection to differential equations and differential inclusions.

\subsection{The QTD Differential Equation}\label{sec:qtd-ode}

Taking the expectation over the random variables $R_k(x)$ and $X'_k(x)$ in Equation~\eqref{eq:qtd-sa} conditional on the algorithm history up to time $k$ yields an expected increment of
\begin{align}\label{eq:expected-increment}
    \alpha_k \left( \tau_i - \mathbb{P}^\pi_x(R + \theta_k(X', J) < \theta_k(x, i) ) \right)\, .
\end{align}
We now briefly introduce an assumption on the MDP reward structure that simplifies the analysis that follows. This assumption guarantees that the two ``difficult'' cases of flat and vertical regions of CDFs (see Figure~\ref{fig:quantiles}) do not arise; note that this assumptions removes the possibility of multiple fixed points or discontinuous expected dynamics, as described in Example~\ref{ex:2d}. We will lift this assumption later.

\begin{assumption}\label{assume:lipschitz-monotonicity}
    For each state $x \in \mathcal{X}$, the reward distribution at $x$ has a CDF which is strictly increasing, and Lipschitz continuous.
\end{assumption}

As described in Section~\ref{sec:dp-analysis}, the distribution of $R + \theta_k(X', J)$ given the initial state $x$ is in fact equal to the application of the distributional Bellman operator $\mathcal{T}^\pi$ applied to the return-distribution function $\eta_k \in \mathscr{P}(\mathbb{R})^\mathcal{X}$ given by
\begin{align*}
    \eta_k(x) = \frac{1}{m} \sum_{i=1}^m \delta_{\theta_k(x, i)} \, .
\end{align*}
Under Assumption~\ref{assume:lipschitz-monotonicity}, and in particular the assumption of continuous reward CDFs, this yields a concise rewriting of the increment as
\begin{align*}
    \alpha_k \left( \tau_i - F_{(\mathcal{T}^\pi \theta_k)(x)}(\theta_k(x, i)) \right) \, .
\end{align*}
We may therefore intuitively interpret Equation~\eqref{eq:qtd-sa} as a noisy discretisation of the differential equation
\begin{align}\label{eq:ode}
    \partial_t \vartheta_t(x, i) = \tau_i - F_{(\mathcal{T}^\pi \vartheta_t)(x)}(\vartheta_t(x, i)) \, ,
\end{align}
which we refer to as the QTD differential equation (or QTD ODE). Note also that Assumption~\ref{assume:lipschitz-monotonicity} guarantees the global existence and uniqueness of solutions to this differential equation, by the Cauchy-Lipschitz theorem.

\begin{remark}
Calling back to Figure~\ref{fig:example2}, the trajectories of the QTD ODE are obtained precisely by integrating the vector fields that appear in these plots.
In contrast to the ODE that emerges when analysing classical TD learning (both in tabular and linear function approximation settings) \citep{tsitsiklis1997analysis}, the right-hand side of Equation~\eqref{eq:ode} is non-linear in the parameters $\vartheta_t$, meaning that we are outside the domain of linear stochastic approximation methods.
\end{remark}

\subsection{The QTD Differential Inclusion}

In lifting Assumption~\ref{assume:lipschitz-monotonicity}, a few complications arise. Firstly, if $F_{(\mathcal{T}^\pi \theta)(x)}$ is not continuous at $\theta(x, i)$, then the right-hand side of the QTD ODE in Equation~\eqref{eq:ode} is modified to 
\begin{align*}
    \tau_i - \mathbb{P}_{Z \sim (\mathcal{T}^\pi \vartheta_t)(x)}(Z < \vartheta_t(x, i)) \, ;
\end{align*}
the difference is the strict inequality. Now the right-hand side of the differential equation itself is not continuous; in general, solutions may not even exist for this differential equation. The situation is illustrated in the bottom-left panel of Figure~\ref{fig:example2}; the lines in this plot illustrate points of discontinuity of the vector field to be integrated, and there are instances where the vector field either side of such a line of discontinuity ``pushes'' back into the discontinuity. In such cases, the differential equation has no solution in the usual sense. This phenomenon is known as sliding, or sticking, from cases when it arises in the modelling of physical systems with potentially discontinuous forces (such as static friction models in mechanics).

\citet{filippov1960differential} proposed a method to deal with such non-existence issues, by relaxing the definition of the dynamics at points of discontinuity. Technically, \citeauthor{filippov1960differential}'s proposal is to allow the derivative to take on any value in the convex hull of possible limiting values as we approach the point of discontinuity. In our case, we consider redefining the dynamics at points of discontinuity as follows:
\begin{align}\label{eq:qtd-di}
    \partial_t \vartheta_t(x, i) \in [\tau_i - F_{(\mathcal{T}^\pi \vartheta_t)(x)}(\vartheta_t(x, i)), \tau_i - F_{(\mathcal{T}^\pi \vartheta_t)(x)}(\vartheta_t(x, i)-)] \, ,
\end{align}
where $F_{(\mathcal{T}^\pi \vartheta_t)(x)}(\vartheta_t(x, i)-)$ denotes $\lim_{s \uparrow \vartheta_t(x, i)} F_{(\mathcal{T}^\pi \vartheta_t)(x)}(s)$. This refines the dynamics so that for each coordinate $(x, i)$, the derivative may take on either the left or right limit around $\vartheta_t(x, i)$, or any value in between; this is a looser relaxation than \citeauthor{filippov1960differential}'s proposal, and is easier to work with in our analysis.

Equation~\eqref{eq:qtd-di} is a \emph{differential inclusion}, as opposed to a differential equation; the derivative is constrained to a set at each instant, rather than constrained to a single value. We refer to Equation~\eqref{eq:qtd-di} specifically as the QTD differential inclusion (or QTD DI). Note that if $F_{(\mathcal{T}^\pi \theta)(x)}$ is continuous at $\theta(x, i)$, then the right-hand side of Equation~\eqref{eq:qtd-di} reduces to the singleton $\{ \tau_i - F_{(\mathcal{T}^\pi \theta)(x)}(\theta(x, i))\}$, and we thus obtain the ODE dynamics considered previously.

\subsection{Solutions of Differential Inclusions}

We briefly recall some key concepts regarding solutions of differential inclusions; a full review of the theory of differential inclusions is beyond the scope of this article, and we refer the reader to the standard references by \citet{aubin1984differential}, \citet{clarke1998nonsmooth}, and \citet{smirnov2002introduction}.

\begin{definition}\label{def:di-soln}
    Let $H : \mathbb{R}^n \rightrightarrows \mathbb{R}^n$ be a set-valued map. 
    The path $(z_t)_{t \geq 0}$ is a solution to the differential inclusion $\partial_t z_t \in H(z_t)$ if there exists an integrable function $g : [0, \infty) \rightarrow \mathbb{R}^{n}$ such that
    \begin{align}\label{eq:integral}
        z_t = \int_0^t g_s \mathrm{d}s
    \end{align}
     for all $t \geq 0$, and $g_t \in H(z_t)$ for almost all $t \geq 0$.
\end{definition}

Note that Definition~\ref{def:di-soln} does not require that $z_t$ is \emph{differentiable} with derivative $g_t$, but only the weaker integration condition in Equation~\eqref{eq:integral}. 
We then have the following existence result (see, for example, \citeauthor{smirnov2002introduction}, \citeyear{smirnov2002introduction} for a proof).

\begin{proposition}
    Consider a set-valued map $H : \mathbb{R}^n \rightrightarrows \mathbb{R}^n$, and suppose that $H$ is a \emph{Marchaud map}: that is,
    \begin{itemize}
        \item the set $\{ (z, h) : z \in \mathbb{R}^n, h \in H(z) \}$ is closed.
        \item For all $z \in \mathbb{R}^n$, $H(z)$ is non-empty, compact, and convex.
        \item There exists a constant $C > 0$ such that for all $z \in \mathbb{R}^n$,
        \begin{align*}
            \max_{h \in H(z)} \| h \| \leq C ( 1 + \|z\| ) \, .
        \end{align*}
    \end{itemize}
    Then the differential inclusion $\partial_t z_t \in H(z_t)$ has a global solution, for any initial condition.
\end{proposition}

It is readily verified that the QTD DI satisfies the requirements of this result, and we are therefore guaranteed global solutions to this differential inclusion, under any initial conditions.

\subsection{Asymptotic Behaviour of Differential Inclusion Trajectories}

Recall that our goal is to show that the trajectories of the QTD differential inclusion must approach the fixed points of QDP. A key tool in doing so is the notion of a Lyapunov function; the following definition is based on \citet{benaim2005stochastic}.

\begin{definition}
    Consider a Marchaud map $H : \mathbb{R}^n \rightrightarrows \mathbb{R}^n$, and a subset $\Lambda \subseteq \mathbb{R}^n$. A continuous function $L : \mathbb{R}^n \rightarrow [0, \infty)$ is said to be a \emph{Lyapunov function} for the differential inclusion $\partial_t z_t \in H(z_t)$ and subset $\Lambda$ if for any solution $(z_t)_{t \geq 0}$ of the differential inclusion and $0 \leq s < t$, we have $L(z_t) < L(z_s)$ for all $z_s \not\in \Lambda$ and $L(z) = 0$ for all $z \in \Lambda$.
\end{definition}

Intuitively, $L$ is a Lyapunov function if it decreases along trajectories of the differential inclusion, and is minimal precisely on $\Lambda$. Lyapunov functions are a central tool in dynamical systems for demonstrating convergence, and in the sections that follow, we will consider the QTD differential inclusion, and take $\Lambda$ to be the set of fixed points of the family of QDP algorithms.

\subsection{QTD as a Stochastic Approximation to the QTD Differential Inclusion}

We can now give the proof of our core result, Theorem~\ref{thm:general-convergence}. The abstract stochastic approximation result at the heart of the convergence proof of QTD is presented below. It is a special case of the general framework described by \citet{benaim2005stochastic}, the proof of which is given in Appendix~\ref{sec:proof-benaim-result}.

\begin{theorem}\label{thm:benaim-result}
    Consider a Marchaud map $H : \mathbb{R}^n \rightrightarrows \mathbb{R}^n$, and the corresponding differential inclusion $\partial_t z_t \in H(z_t)$. Suppose there exists a Lyapunov function $L$ for this differential inclusion and a subset $\Lambda \subseteq \mathbb{R}^n$. Suppose also that we have a sequence $(\theta_k)_{k \geq 0}$ satisfying
    \begin{align*}
        \theta_{k+1} = \theta_k + \alpha_k (g_k + w_k) \, ,
    \end{align*}
    where:
    \begin{itemize}
        \item $(\alpha_k)_{k= 0}^\infty$ satisfy the conditions $\sum_{k=0}^\infty \alpha_k = \infty$, $\alpha_k = o(1/\log(k))$;
        \item $g_k \in H(\theta_k)$ for all $k \geq 0$;
        \item $(w_k)_{k=0}^\infty$ is a bounded martingale difference sequence with respect to the natural filtration generated by $(\theta_k)_{k=0}^\infty$; that is, there is an absolute constant $C$ such that $\|w_k\|_\infty < C$ almost surely, and $\mathbb{E}[w_k | \theta_{0},\ldots,\theta_{k} ] = 0$.
    \end{itemize}
    If further $(\theta_k)_{k=0}^\infty$ is bounded almost surely (that is, $\sup_{k \geq 0}\| \theta_k \|_\infty < \infty$ almost surely), then $\theta_k \rightarrow \Lambda$ almost surely.
\end{theorem}

The intuition behind the conditions of the theorem are as follows. The Marchaud map condition ensures the differential inclusion of interest has global solutions. The existence of the Lyapunov function guarantees that trajectories of the differential inclusion converge in a suitably stable sense to $\Lambda$. The step size conditions, martingale difference condition, and boundedness conditions mean that the iterates $(\theta_k)_{k=0}^\infty$ will closely track the differential inclusion trajectories, and hence exhibit the same asymptotic behaviour. We can now give the proof of Theorem~\ref{thm:general-convergence}, first requiring the following proposition, which is proven in Appendix~\ref{sec:proof-bounded}.

\begin{proposition}\label{prop:bounded}
    Under the conditions of Theorem~\ref{thm:general-convergence}, the iterates $(\theta_k)_{k=0}^\infty$ are bounded almost surely.
\end{proposition}

\begin{proof}(Proof of Theorem~\ref{thm:general-convergence})
    We see that for the QTD sequence $(\theta_k)_{k=0}^\infty$ and the QTD DI and QDP invariant set $\Lambda = \{ \qdpfixedpointlambdaparams : \lambda \in [0,1]^{\mathcal{X} \times [m]} \}$, the conditions of Theorem~\ref{thm:benaim-result} are satisfied, except perhaps for the boundedness of $(\theta_k)_{k=0}^\infty$, and the existence of the Lyapunov function.
    The fact that the sequence $(\theta_k)_{k=0}^\infty$ is bounded almost surely is Proposition~\ref{prop:bounded}; its proof is somewhat technical, and given in the appendix. The construction of a valid Lyapunov function is given in Proposition~\ref{prop:general-lyapunov} below, which completes the proof.
\end{proof}

\begin{remark}
    What makes the relaxation to the differential inclusion work in this analysis? We have already seen that some kind of relaxation of the dynamics is required in order to define a valid continuous-time dynamical system; the original ODE may not have solutions in general. If we relax the dynamics too much (an extreme example would be the differential inclusion $\vartheta_t(x, i) \in \mathbb{R}$), what goes wrong? The answer is that there are too many resulting solutions, which do not exhibit the desired asymptotic behaviour. Thus, the differential inclusion in Equation~\eqref{eq:qtd-di} is in some sense just the right level of relaxation of the differential equation we started with, since trajectories of the QTD DI are still guaranteed to converge to the QDP fixed points.
\end{remark}

\subsection{A Lyapunov Function for the QDP Fixed Points}

In this section, we prove the existence of a Lyapunov function required in order to use Theorem~\ref{thm:benaim-result} to prove Theorem~\ref{thm:general-convergence}. We treat the case when Assumption~\ref{assume:lipschitz-monotonicity} holds separately as the proof is instructive, and considerably simpler than the general case. Under this assumption, note that all projections $\Pi^{\lambda}$ behave identically on the image of $\mathcal{T}^\pi$, since all resulting CDFs are strictly increasing. We therefore introduce the notation $\Pi$ to refer to any such projection in this case, and the notation $\hat{\theta}^\pi_m$ to refer to the unique fixed point of $\Pi \mathcal{T}^\pi$.

\begin{proposition}\label{prop:lyapunov}
    Consider the ODE in Equation~\eqref{eq:ode}, and suppose Assumption~\ref{assume:lipschitz-monotonicity} holds. A Lyapunov function for the equilibrium point $\hat{\theta}^\pi_m$ is given by
    \begin{align*}
        L(\theta) = \max_{x \in \mathcal{X}} \max_{i=1,\ldots,m} | \theta(x, i) - \hat{\theta}^\pi_m(x, i) | \, .
    \end{align*}
\end{proposition}

\begin{proof}
    We immediately observe that $L$ is continuous, non-negative, and takes on the value $0$ only at $\hat{\theta}^\pi_m$.
    To show that $L(\vartheta_t)$ is decreasing, where $(\vartheta_t)_{t \geq 0}$ is an ODE trajectory, suppose $(x, i)$ is a state-index pair attaining the maximum in $L(\vartheta_t)$. It is sufficient to show that $\vartheta_t(x, i)$ is moving towards $\hat{\theta}^\pi_m(x, i)$, or expressed mathematically,
    \begin{align*}
        \partial_t \vartheta_t(x, i) \signequals \hat{\theta}^\pi_m(x, i) - \vartheta_t(x, i) \, ,
    \end{align*}
    where we use $a \signequals b$ as shorthand for \emph{equality of signs} $\text{sign}(a) = \text{sign}(b)$, where
    \begin{align*}
        \text{sign}(z) =
        \begin{cases}
            1 & \text{ if } z > 0  \, , \\
            0 & \text{ if } z = 0  \, , \\
            -1 & \text{ if } z < 0  \, . \\
        \end{cases}
    \end{align*}
    Now note that
    \begin{align*}
        \partial_t \vartheta_t(x, i)& = \tau_i - F_{(\mathcal{T}^\pi \vartheta_t)(x)}(\vartheta_t(x, i))\\
        & \signequals F^{-1}_{(\mathcal{T}^\pi \vartheta_t)(x)}(\tau_i) - F^{-1}_{(\mathcal{T}^\pi \vartheta_t)(x)}(F_{(\mathcal{T}^\pi \vartheta_t)(x)}(\vartheta_t(x, i)))\\
        & = (\Pi \mathcal{T}^\pi \vartheta_t)(x, i) - \vartheta_t(x, i) \, ,
    \end{align*}
    where the sign equality follows from Assumption~\ref{assume:lipschitz-monotonicity}; since all reward CDFs are strictly increasing, so too is $F_{(\mathcal{T}^\pi \vartheta_t)(x)}$, and so $F^{-1}_{(\mathcal{T}^\pi \vartheta_t)(x)}$ is strictly monotonic.
    Additionally, from the contractivity of $\Pi \mathcal{T}^\pi$ with respect to $\overline{w}_\infty$ (see Proposition~\ref{prop:contraction}), we have
    \begin{align}
        |(\Pi \mathcal{T}^\pi \vartheta_t)(x, i) - \hat{\theta}^\pi_m(x, i)| & \leq \overline{w}_\infty(\Pi \mathcal{T}^\pi \vartheta_t, \Pi \mathcal{T}^\pi \hat{\theta}^\pi_m)  \nonumber \\
        & \leq \gamma \max_{y \in \mathcal{X}} \max_{j=1,\ldots,m} | \vartheta_t(y, j) - \hat{\theta}^\pi_m(y, j) | \nonumber \\
        & = \gamma |\vartheta_t(x, i) - \hat{\theta}^\pi_m(x, i)| \, ; \label{eq:ineq}
    \end{align}
    the equality follows since we selected $(x, i)$ attain the maximum in the definition of $L(\vartheta_t)$. 
    From this, we deduce
    \begin{align*}
        (\Pi \mathcal{T}^\pi \vartheta_t)(x, i) - \vartheta_t(x, i) \signequals \hat{\theta}^\pi_m(x, i) - \vartheta_t(x, i) \, ;
    \end{align*}
    which follows by considering the three cases for the sign of $\hat{\theta}^\pi_m(x, i) - \vartheta_t(x, i)$. If the sign equals zero, then since $(x, i)$ was chosen to be maximal in the definition of $L(\vartheta_t)$, we have $\vartheta_t = \hat{\theta}^\pi_m$, and hence $\Pi \mathcal{T}^\pi \vartheta_t = \hat{\theta}^\pi_m$, and the claim follows; both sides are equal to 0. For the case $\hat{\theta}^\pi_m(x, i) - \vartheta_t(x, i) < 0$, then note we have
    \begin{align*}
        (\Pi \mathcal{T}^\pi \vartheta_t)(x, i) - \vartheta_t(x, i) &
            \leq \hat{\theta}^\pi_m(x, i) + \gamma (\vartheta_t(x, i) - \hat{\theta}^\pi_m(x,i)) - \vartheta_t(x, i) \\
    &  = (1-\gamma) (\hat{\theta}^\pi_m(x, i) - \vartheta_t(x, i)) < 0 \, ,
    \end{align*}
    as required, with the inequality above following from Equation~\eqref{eq:ineq}. The case $\hat{\theta}^\pi_m(x, i) - \vartheta_t(x, i) > 0$ follow similarly.
    We therefore have
    \begin{align*}
        \partial_t \vartheta_t(x, i) \signequals \hat{\theta}^\pi_m(x, i) - \vartheta_t(x, i) \, .
    \end{align*}
    We therefore have that $L(\vartheta_t)$ is decreasing at $t$, strictly so if $\vartheta_t \not= \hat{\theta}^\pi_m$, as required to establish the result.
\end{proof}

The proof of Proposition~\ref{prop:lyapunov} also sheds further light on the mechanisms underlying the QTD algorithm. A key step in the argument is to show that for the state-index pairs $(x, i)$ such that $\vartheta_t(x, i)$ is maximally distant from the fixed point $\theta^\pi_m(x, i)$, the expected update under QTD moves this coordinate of the estimate in the same direction as gradient descent on a squared loss from the fixed point. However, the fact that it is only the \emph{sign} of the update that has this property, and not its magnitude, means that the empirical rate of convergence and stability of QTD can be expected to be somewhat different from methods based on an $L^2$ loss, such as classical TD.

We now state the Lyapunov result in the general case; the proof is somewhat more involved, and is given in Appendix~\ref{sec:proof-general-lyapunov}.

\begin{proposition}\label{prop:general-lyapunov}
    The function
    \begin{align}\label{eq:non-unique-lyapunov}
        L(\theta)  = \min_{\lambda  \in [0,1]^{\mathcal{X} \times [m]}} \max_{(x, i)} | \theta(x, i) - \hat{\theta}^{\pi}_{\lambda}(x, i)  |
    \end{align}
    is a Lyapunov function for the differential inclusion in Equation~\eqref{eq:qtd-di} and the set of fixed points $\{ \hat{\theta}^{\pi}_{\lambda} : \lambda \in [0,1]^{\mathcal{X} \times [m]} \}$.
\end{proposition}

\subsection{Extension to Asynchronous QTD}
\label{sec:main-async}

Our convergence results have focused on the synchronous case of QTD. However, in practice, it is often of interest to implement \emph{asynchronous} versions of TD algorithms, in which only a single state is updated at a time. More formally, an asynchronous version of QTD computes the sequence $(\theta_k)_{k \geq 0}$ defined by an initial estimate $\theta \in \mathbb{R}^{\mathcal{X} \times [m]}$, a sequence of transitions
$(X_k, R_k, X'_k)_{k \geq 0}$, and the update rule
\begin{align*}
     \theta_{k+1}(x, i) =
        \theta_k(x, i) + 
        \beta_{x, k} \frac{1}{m} \sum_{j=1}^m (\tau_i - \mathbbm{1}\{ R_k + \gamma \theta_k(X'_k, j) < \theta_k(x, i)  \})
\end{align*}
for $x = X_k$, and $\theta_{k+1}(x, i) = \theta_k(x, i)$ otherwise. 
Here, the step size $\beta_{x, k}$ depends on both $x$ and $k$, and is typically selected so that each state individually makes use of a fixed step size sequence $(\alpha_k)_{k =0}^\infty$, by taking $\beta_{x, k} = \alpha_{\sum_{l=0}^k \mathbbm{1}\{ X_l = x \} }$. 
This models the online situation where a stream of experience $(X_k, R_k)_{k \geq 0}$ is generated by interacting with the environment using the policy $\pi$, and updates are performed setting $X'_k = X_{k+1}$, and also the setting in which the tuples $(X_k, R_k, X'_k)_{k \geq 0}$ are sampled i.i.d.\ from a replay buffer, among others.

Convergence of QTD in such asynchronous settings can also be proven; \citet{perkins2013asynchronous} extend the analysis of \citet{benaim1999dynamics} and \citet{benaim2005stochastic}, incorporating the approach of \citet{borkar1998asynchronous}, to obtain convergence guarantees for asynchronous stochastic approximation algorithms approximating differential inclusions. In the interest of space, we do not provide the full details of the proof here, but instead sketch the key differences that arise in the analysis in Appendix~\ref{sec:async}.

\section{Analysis of the QTD Limit Points}
\label{sec:fixed-point-analysis}

In general, the limiting points $\qdpfixedpointlambda$ for QTD/QDP will not be the same as the true return-distribution function $\eta^\pi$. On the one hand, this is clear; each return-distribution function $\qdpfixedpointlambda$ is in the image of the projection $\Pi^\lambda$, so each constituent probability distribution must be of the form $\tfrac{1}{m} \sum_{i=1}^m \delta_{z_i}$, whereas the true return distributions need not take on this form. In addition, the magnitude of this approximation error is not immediately clear. Each application of the projection $\Pi^\lambda$ in the dynamic programming process causes some loss of information, and the quality of the fixed point $\qdpfixedpointlambda$ is affected by the build up of these approximations over time. 

Measuring approximation error in $\bar{w}_\infty$ typically turns out to be uninformative, as $\bar{w}_\infty$ is a particularly strict notion of distance between probability distributions, as discussed in the context of distributional RL by \citet{rowland2019statistics} and \citet{bdr2022}. In particular, fixed points $\qdpfixedpointlambda$ that intuitively provide a good approximation to $\eta^\pi$ may have high $\bar{w}_\infty$-distance, and the $\bar{w}_\infty$-distance generally does not decrease with $m$ \citep{bdr2022}. Instead, we use the Wasserstein-1 metric, and its extension to return-distribution functions, defined by
\begin{align*}
    w_1(\nu, \nu') = \int_0^1 |F^{-1}_\nu(t)  - F^{-1}_{\nu'}(t) | \; \mathrm{d}t \, ,  \quad \bar{w}_1(\eta, \eta') = \max_{x \in \mathcal{X}} w_1(\eta(x), \eta'(x)) \, ,
\end{align*}
for all $\nu, \nu' \in \mathscr{P}(\mathbb{R})$, and $\eta, \eta' \in \mathscr{P}(\mathbb{R})^\mathcal{X}$.
The following result improves on the analysis given by \citet{bdr2022} for the case of $\lambda \equiv 0$, establishing an upper bound on the $\overline{w}_1$ distance between $\qdpfixedpointlambda$ and $\eta^\pi$ for any $\lambda$, essentially by showing that the errors accumulated in dynamic programming can be made arbitrarily small by increasing $m$, which controls the richness of the distribution representation.

\begin{restatable}{proposition}{propFixedPointQualityOne}\label{prop:fixed-point-quality-1}
    For any $\lambda \in [0,1]^{\mathcal{X} \times [m]}$, if all reward distributions are supported on $[R_\textsc{min}, R_\textsc{max}]$, then we have
    \begin{align*}
        \bar{w}_1(\qdpfixedpointlambda, \eta^\pi) \leq \frac{V_\textsc{max} - V_{\textsc{min}}}{2m(1-\gamma)} \, ,
    \end{align*}
    where $V_\textsc{max}=R_\textsc{max}/(1-\gamma)$, and similarly $V_\textsc{min}=R_\textsc{min}/(1-\gamma)$.
\end{restatable}

\begin{remark}
    This bound also provides motivation for the specific values of $(\tau_i)_{i=1}^m$ that QTD uses. A similar convergence analysis and fixed-point analysis can be straightforwardly carried out for a version of the QTD algorithm with other values for $(\tau_i)_{i=1}^m$; by tracing through the proof of Proposition~\ref{prop:fixed-point-quality-1}, it can be seen that the bound is proportional to $\max(\tau_1, \max((\tau_{i+1} - \tau_i)/2 : i=2,\ldots,m-1), 1-\tau_m)$, which is minimised precisely by the choice of $(\tau_i)_{i=1}^m$ used by QTD.
\end{remark}

\subsection{Instance-Dependent Bounds}

The result above implicitly assumes the worst-case projection error is incurred at all states with each application of the Bellman operator. In environments where this is not the case, the fixed point can be shown to be of considerably better quality. We describe an example of an instance-dependent quality bound here.

\begin{proposition}\label{prop:improved-fp-bound}
Consider an MDP such that for any trajectory, after $k$ time steps all encountered transition distributions and reward distributions are Dirac deltas. If all reward distributions in the MDP are supported on $[R_\textsc{min}, R_\textsc{max}]$, then for any $\lambda \in [0,1]^{\mathcal{X} \times [m]}$, we have
    \begin{align*}
        \bar{w}_1(\qdpfixedpointlambda, \eta^\pi) \leq \frac{(V_\textsc{max} - V_{\textsc{min}})(1-\gamma^k)}{2m(1-\gamma)} \, .
    \end{align*}
\end{proposition}

\begin{remark}
    One particular upshot of this bound for practitioners is that for agents in near-deterministic environments using near-deterministic policies, it may be possible to use $m=o((1-\gamma)^{-1})$ quantiles and still obtain accurate approximations to the return-distribution function via QTD and/or QDP. It is interesting to contrast this result for quantile-based distributional reinforcement learning against the case when using categorical distribution representations \citep{bellemare2017distributional,rowland2018analysis,bdr2022}. In this latter case, fixed point error continues to be accumulated even when the environment has solely deterministic transitions and rewards, due to the well-documented phenomenon of the approximate distribution `spreading its mass out' under the Cram\'er projection \citep{bellemare2017distributional,rowland2018analysis,bdr2022}. Our observation here leads to immediate practical advice for practitioners (in environments with mostly deterministic transitions, a quantile representation may be preferred to a categorical representation, leading to less approximation error), and raises a general question that warrants further study: how can we use prior knowledge about the structure of the environment to select a good distribution representation?
\end{remark}

We conclude this section by noting that many variants of Proposition~\ref{prop:improved-fp-bound} are possible; one can for example modify the assumption that rewards are deterministic to an assumption that rewards distributions are supported on a `small' interval, and still obtain a fixed-point bound that improves over the instance-independent bound of Proposition~\ref{prop:fixed-point-quality-1}.
There are a wide variety of such modifications that could be imagined, and we believe this to be an interesting direction for future research and applications.

\subsection{Qualitative Analysis of QDP Fixed Points}

The analysis in the previous section establishes quantitative upper bounds on the quality of the fixed point learnt by QTD, and guarantees that with enough atoms an arbitrarily accurate approximation of the return-distribution function (as measured by $w_1$) can be learnt. We now take a closer look at the way in which approximation errors may manifest in QTD and QDP.

\begin{example}\label{ex:fixed-points}
    Consider the two-state Markov decision process (with a single action) whose transition probabilities are specified by the left-hand side of Figure~\ref{fig:qualitative-example}, such that a deterministic reward of $2$ is obtained in state $x_1$, and $-1$ in state $x_2$; further, let us take a discount factor $\gamma = 0.9$. The centre panel of this figure shows various estimates of the CDF for the return distribution at state $x_1$. The ground truth estimate in black is obtained from Monte Carlo sampling. The CDFs in purple, blue, green, and orange are the points of convergence for QDP with $m=2, 5, 10, 100$, respectively. For $m=100$, a very close fit to the true return distribution is obtained. However, for small $m$ in particular, the distribution is heavily skewed to the right. In the case of $m=2$, half of the probability mass is placed on the greatest possible return in this MDP---namely 20---even though with probability 1 the true return is less than this value. What is the cause of this behaviour from QDP? This question is answered by investigating the dynamic programming update itself in more detail.
    
    In this MDP, the result of the QDP operator applied to the fixed point $\theta$ is to update each particle location with a `backed-up' particle appearing in the distributions $\mathcal{T}^\pi \theta$.
    When such settings arise, tracking which backed-up particles are allocated to which other particles helps us to understand the behaviour of QDP, and the nature of the approximation incurred. We also gain intuition about the situation, since the QDP operator is behaving like a an affine policy evaluation operator on $\mathcal{X} \times [m]$ \emph{locally} around the fixed point.
    We can visualise which particles are assigned to one another by a QDP operator application through \emph{local quantile back-up diagrams};
    the right-hand side of Figure~\ref{fig:qualitative-example} shows the local quantile back-up diagram for particular MDP. We observe that $\theta(x_1, 2)$ backs up from itself, and hence learns a value that corresponds to observing a self-transition at every state, with a reward of 2; under the discount factor of $0.9$, this corresponds to a return of 20. This is the source of the drastic over-estimation of returns in the approximation obtained with $m=2$, and the fact that all other state-quantile pairs implicitly bootstrap from this estimate leads to the over-estimation leaking out into all quantiles estimated in this case. As $m$ increases, we observe from the CDF plot that there is always one particle that learns this maximal return of 20, but that this has less effect on the other quantiles; indeed in the orange curve, we obtain a very good approximation (in $w_1$) to the true return distribution despite this particle with a maximal value of 20 remaining present. We can interpret the increase in $m$ as preventing pathological self-loops/small cycles in the quantile backup diagram from ``leaking out'' and degrading the quality of other quantile estimates; this provides a complementary perspective on the approximation artefacts that occur in QDP/QTD fixed points to the quantitative upper bounds in the previous section.
\end{example}

We expect the local quantile back-up diagram introduced in Example~\ref{ex:fixed-points} to be a useful tool for developing intuition, as well as further analysis, of QDP and QTD. As described in the example itself, being able to define the local back-up diagram depends on the structure of the MDP being such that the QDP operator obtains each new coordinate value from a single backed-up particle location. It is an interesting question as to how the definition of such local back-up diagrams could be generalised to apply in situations where this does not hold, such as with absolutely continuous reward distributions.

\begin{figure}
    \centering
    \null\hfill
    \includegraphics[keepaspectratio,width=.3\textwidth,valign=c]{img/qualitaitive_analysis_mdp.pdf}\hfill
    \includegraphics[keepaspectratio,width=.3\textwidth,valign=c]{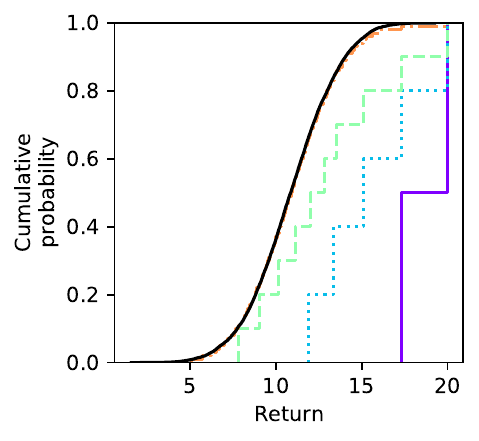}
    \hfill
    \includegraphics[keepaspectratio,width=.3\textwidth,valign=c]{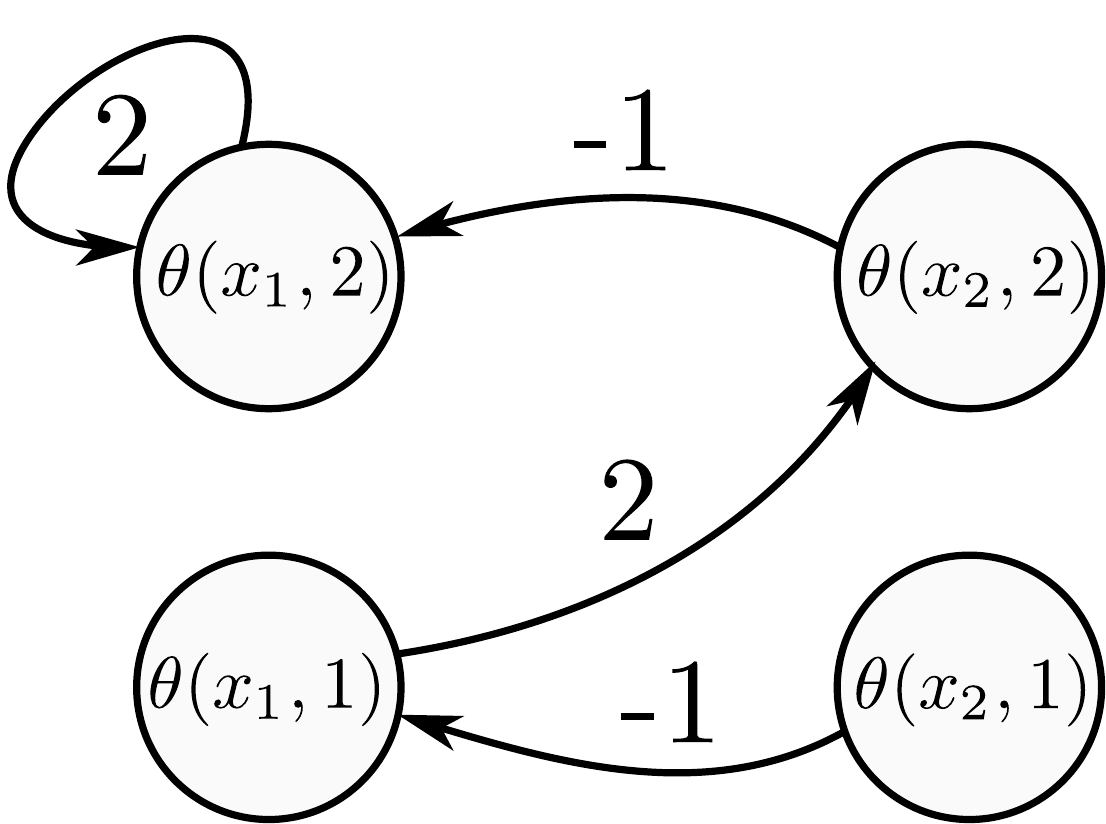}
    \hfill\null
    \caption{
    Left: An example MDP.
    Centre: The fixed point return distribution estimates for state $x_1$ obtained by QDP for $m=2,5,20,100$ (solid purple, dotted blue, dashed green, and dash-dotted orange, respectively) compared to ground truth in solid black.
    Right: The corresponding local quantile backup diagram at the fixed point for $m=2$, illustrating potential approximation artefacts in QDP fixed points.}
    \label{fig:qualitative-example}
\end{figure}

\section{Related Work}
\label{sec:extensions}

\emph{Stochastic approximation theory with differential inclusions.} The ODE method was introduced by \citet{ljung1977analysis} as a means of analysing stochastic approximation algorithms, and was subsequently extended and refined by \citet{kushner1978}; standard references on the subject include \citet{kushner2003stochastic,borkar2009stochastic,benveniste2012adaptive}; see also \citet{meyn2022control} for an overview in the context of reinforcement learning. The framework we follow in this paper is set out by \citet{benaim1999dynamics}, and was extended by \citet{benaim2005stochastic} to allow for differential inclusions. \citet{perkins2013asynchronous} later extended this analysis further to allow for asynchronous algorithms, building on the approach introduced by \citet{borkar1998asynchronous}, and extended, with particular application to reinforcement learning, by \citet{borkar2000ode}.

\emph{Differential inclusion theory.} Differential inclusions have found application across a wide variety of fields, including control theory \citep{wazewski1961systemes}, economics \citep{aubin1991viability} differential game theory \citep{krasovskii1988game}, and mechanics \citep{marques2013differential}. The approach to modelling differential equations with discontinuous right-hand sides via differential inclusions was introduced by \citet{filippov1960differential}. Standard references on the theory of differential inclusions include \citet{aubin1984differential,clarke1998nonsmooth,smirnov2002introduction}; see also \citet{bernardo2008piecewise} on the related field of piecewise-smooth dynamical systems. \citet{joseph2019adaptive} also use tools combining stochastic approximation and differential inclusions from \citet{benaim2005stochastic} to analyse (sub-)gradient descent as a means of estimating quantiles of fixed distributions. Within reinforcement learning and related fields more specifically, differential inclusions have played a key role in the analysis of game-theoretic algorithms based on fictitious play \citep{brown1951iterative,robinson1951iterative}; see \citet{benaim2006stochastic,leslie2006generalised,benaim2013consistency} for examples.
More recently, \citet{gopalan2022approximate} used differential inclusions to analyse TD algorithms for control with linear function approximation.

\emph{Quantile regression.} Quantile regression as a methodology for statistical inference was introduced by \cite{koenker1978regression}. \citet{koenker2005quantile} and \citet{koenker2017handbook} provide detailed surveys of the field. Quantile temporal-difference learning may be viewed as fusing quantile regression with the bootstrapping approach (\emph{learning a guess from a guess}, as \citet{sutton2018reinforcement} express it) that is core to much of the reinforcement learning methodology.

\emph{Quantiles in reinforcement learning.} The approach to distributional reinforcement learning based on quantiles was introduced by \citet{dabney2018distributional}. A variety of modifications and extensions were then considered in the deep reinforcement learning setting \citep{dabney2018implicit,yang2019fully,zhou2020non,luo2021distributional}, as well as further developments on the theoretical side \citep{lheritier2022cramer}. A summary of the approach is presented by \citet{bdr2022}. \citet{gilbert2016quantile} study the problem of optimising quantile criteria in end-state MDPs. \citet{li2022quantile} consider the risk-sensitive control problem of optimising particular quantiles of the return distribution, and derive a dynamic programming algorithm that maintains a value function over state and the ``target quantile-to-go''.

\section{Conclusion}

We have provided the first convergence analysis for QTD, a popular and effective distributional reinforcement learning algorithm. In contrast to the analysis of many classical temporal-difference learning algorithms, this has required the use of tools from the field of differential inclusions and branches of stochastic approximation theory that deal with the associated dynamical systems.
Due to the structure of the QTD algorithm, such as its bounded-magnitude updates, these convergence guarantees hold under weaker conditions than are generally used in the analysis of TD algorithms.
These results establish the soundness of QTD, representing an important step towards understanding its efficacy and practical successes, and we expect the theoretical tools used here to be useful in further analyses of (distributional) reinforcement learning algorithms.

There are several natural directions for further work building on this analysis. One such direction is to establish finite-sample bounds for the convergence of QTD predictions to the set of QDP fixed points. 
This is a central theoretical question for developing our understanding of QTD, and may also shed further light on the recently observed empirical phenomenon in which tabular QTD can outperform TD in stochastic environments as a means of value estimation \citep{rowland2023statistical}. Related to this point, the Lyapunov analysis conducted in this paper provides further intuition for \emph{why} QTD works in general, and we expect this to inform the design of further variants of QTD, for example incorporating multi-step bootstrapping \citep{watkins1989learning}, or Ruppert-Polyak averaging \citep{ruppert1988efficient,polyak1992acceleration}.
Another important direction is to analyse more complex variants of the QTD algorithm, incorporating more aspects of the large-scale systems in which it has found application. Examples include incorporating function approximation, or control variants of the algorithm based on Q-learning. We believe further research into the theory, practice and applications of QTD, at a variety of scales, are important directions for foundational reinforcement learning research.


\acks{%
We thank the anonymous reviewers for helpful suggestions and feedback. We also thank Tor Lattimore for detailed comments on an earlier draft, and David Abel, Bernardo Avila Pires, Diana Borsa, Yash Chandak, Daniel Guo, Clare Lyle, and Shantanu Thakoor for helpful discussions.
Marc G. Bellemare was supported by Canada CIFAR AI Chair funding.
The simulations in this paper were generated using the Python 3 language, and made use of the NumPy \citep{harris2020array}, SciPy \citep{2020SciPy-NMeth}, and Matplotlib \citep{hunter2007matplotlib} libraries.
}

\appendix

\section{Proofs}

In this section, we provide proofs for results which are not proven in the main text.

\subsection{Proof of Proposition~\ref{prop:proj-nonexpansion}}\label{sec:proof-proj-nonexpansion}

Let $\eta, \eta' \in \mathscr{P}(\mathbb{R})$. We have
\begin{align*}
    & \Big|\left((1-\lambda(x, i)) F^{-1}_{\eta(x)}(\tau_i) + \lambda(x, i) \bar{F}^{-1}_{\eta(x)}(\tau_i)\right)
    \!-\!
    \left( (1-\lambda(x, i)) F^{-1}_{\eta'(x)}(\tau_i) + \lambda(x, i) \bar{F}^{-1}_{\eta'(x)}(\tau_i)\right)
    \Big| \\
    & \qquad\qquad \leq  (1-\lambda(x, i)) |F^{-1}_{\eta(x)}(\tau_i) - F^{-1}_{\eta'(x)}(\tau_i) | +  \lambda(x, i) | \bar{F}^{-1}_{\eta(x)}(\tau_i) - \bar{F}^{-1}_{\eta'(x)}(\tau_i)| \, .
\end{align*}
Clearly 
\begin{align*}
    |F^{-1}_{\eta(x)}(\tau_i) - F^{-1}_{\eta'(x)}(\tau_i) | \leq \bar{w}_\infty(\eta, \eta') \, .
\end{align*}
Additionally, we have
\begin{align*}
    |\bar{F}^{-1}_{\eta(x)}(\tau_i) - \bar{F}^{-1}_{\eta'(x)}(\tau_i) | & = |\lim_{s \downarrow \tau_i} F^{-1}_{\eta(x)}(s) - \lim_{s \downarrow \tau_i} F^{-1}_{\eta'(x)}(s) |\\
    &= \lim_{s \downarrow \tau_i} | F^{-1}_{\eta(x)}(s) -  F^{-1}_{\eta'(x)}(s) | \\
    & \leq \bar{w}_\infty(\eta, \eta') \, .
\end{align*}
Putting this together, we obtain
\begin{align*}
    & \Big|\left((1-\lambda(x, i)) F^{-1}_{\eta(x)}(\tau_i) + \lambda(x, i) \bar{F}^{-1}_{\eta(x)}(\tau_i)\right) \!-\!
    \left((1-\lambda(x, i)) F^{-1}_{\eta'(x)}(\tau_i) + \lambda(x, i) \bar{F}^{-1}_{\eta'(x)}(\tau_i)\right) \Big| \\
   & \qquad \leq \bar{w}_\infty(\eta, \eta')  \, ,
\end{align*}
as required.

\subsection{Proof of Theorem~\ref{thm:benaim-result}}\label{sec:proof-benaim-result}

Theorem~\ref{thm:benaim-result} is essentially a special case of the general results presented in \citet{benaim2005stochastic}, in the form needed for the proof of convergence of QTD. To explain how to obtain Theorem~\ref{thm:benaim-result} from the results of \citet{benaim2005stochastic}, first, we associate a continuous-time path $(\bar{\theta}(t))_{t \geq 0}$ with the iterates $(\theta_k)_{k=0}^\infty$ by linear interpolation, in particular defining $\bar{\theta}(\sum_{k=0}^s \alpha_k) = \theta_s$, and linearly interpolating in between. The continuous-time path $(\bar{\theta}(t))_{t \geq 0}$ satisfies the definition of a \emph{perturbed solution} of the Marchaud differential inclusion with probability 1, as defined by Definition II of \citet{benaim2005stochastic}, since: (i) $\bar{\theta}$ is piecewise linear, hence absolutely continuous; (ii) the difference $\| \theta_{k+1} - \theta_k \|_\infty$ is $O(\alpha_k)$, due to the growth condition on $H$ and since $\bar{\theta}$ is bounded by assumption; and (iii) the lim-sup condition holds with probability 1 thanks to the boundedness of the martingale difference sequence $(w_k)_{k=0}^\infty$ and Proposition~1.4 of \citet{benaim2005stochastic}; see also Theorem~5.3.3 of \citet{kushner2003stochastic}.

Next, since we assume $\bar{\theta}$ is bounded, Theorem~4.2 of \citet{benaim2005stochastic} applies so that we deduce that it is an asymptotic pseudotrajectory of the differential inclusion (w.p.1). We then have that $(\bar{\theta}(t))_{t \geq 0}$ is a bounded asymptotic pseudotrajectory (w.p.1), so Theorem~4.3 of \citet{benaim2005stochastic} applies, and we deduce that the set of limit points of $(\bar{\theta}(t))_{t \geq 0}$ is internally chain transitive (w.p.1). But now by Proposition~3.27 of \citet{benaim2005stochastic} applied to the Lyapunov function $L$ and the set $\Lambda$, all internally chain transitive sets are contained within $\Lambda$. Since $(\bar{\theta}(t))_{t \geq 0}$ is bounded, we deduce that it converges to $\Lambda$ (w.p.1). It therefore follows that the discrete sequence $(\theta_k)_{k=0}^\infty$ converges to $\Lambda$ with probability 1, as required.

\subsection{Proof of Proposition~\ref{prop:bounded}}\label{sec:proof-bounded}

Roughly, the intuition of the proof is that the structure of the QTD differential inclusion means that when $\|\theta_k\|_\infty$ is sufficiently large, the coordinates of $\theta_k$ furthest from the origin are moved back towards the origin by the differential inclusion. We then argue that the martingale noise cannot cause divergence, which completes the argument.

\emph{Differential inclusion update direction.} 
To begin with the analysis of the differential inclusion, fix $\delta > 0$ such that $1-\delta > \gamma$, and let $M > 0$ be such that for all $(x, a) \in \mathcal{X} \times \mathcal{A}$, we have
\begin{align*}
    F_{P_{\mathcal{R}}(x, a)}((1-\delta - \gamma)M) > 1 - 1/(4m) \, , \quad F_{P_{\mathcal{R}}(x, a)}(-(1-\delta - \gamma)M) < 1/(4m) \, .
\end{align*}
We then introduce the events 
\begin{align*}
    I^+_k(x, i) & = \{ \| \theta_k \|_\infty > M \, , \ \ \theta_k(x, i) > (1-\delta) \| \theta_k \|_\infty \} \, , \\
    I^-_k(x, i) & = \{ \| \theta_k \|_\infty > M \, , \ \ \theta_k(x, i) < - (1-\delta) \| \theta_k \|_\infty \} \, .
\end{align*}
which, roughly speaking, hold when $\theta_k$ has at least one large coordinate (in absolute value), and $\theta_k(x, i)$ is a positive (respectively, negative) coordinate close to the maximum value. When $I^+_k(x, i)$ holds, we have
\begin{align}
    & \phantom{=} \tau_i - F_{(\mathcal{T}^\pi \theta_k)(x)}(\theta_k(x, i)-)\\
    & = \tau_i - \sum_{x' \in \mathcal{X}} P(x'|x,a)\pi(a|x) \frac{1}{m} \sum_{j=1}^m \lim_{s \uparrow 0} F_{\mathcal{R}(x, a)}(\theta_k(x, i) - \gamma \theta_k(x', j) + s)\nonumber \\
    & \overset{(a)}{\leq} \tau_i - \sum_{x' \in \mathcal{X}} P(x'|x,a)\pi(a|x) \frac{1}{m} \sum_{j=1}^m (1 - 1/(4m)) \nonumber\\
    & \leq (1 - 1/(2m)) - (1 - (1/(4m)) \nonumber \\
    & \leq - 1/(4m) \, , \label{eq:upper-bound}
\end{align}
and hence the differential inclusion moves $\theta_k(x, i)$ towards the origin. Inequality (a) follows since on $I^+_k(x, i)$, we have
\begin{align*}
    \theta_k(x, i) - \gamma \theta_k(x', j) \geq (1-\delta) \|\theta_k\|_\infty - \gamma \|\theta_k\|_\infty \geq (1 - \delta - \gamma) M \, .
\end{align*}
Analogously, we conclude that on $I^-_k(x, i)$, we have
\begin{align*}
    \tau_i - F_{(\mathcal{T}^\pi \theta_k)(x)}(\theta_k(x, i)) \geq 1/(4m) \, ,
\end{align*}
and so the differential inclusion moves $\theta_k(x, i)$ towards the origin in this case too.

\emph{Chaining updates and reasoning about noise.}
To describe the relationship between successive iterates in the sequence $(\theta_k)_{k \geq 0}$, we introduce the notation $\theta_{k+1} = \theta_k + \alpha_k g_k + \alpha_k w_k$, where $w_j$ is martingale difference noise, and hence $g_j$ is an expected update direction, from the right-hand side of the QTD differential inclusion. By boundedness of the update noise and the step size assumptions, we have from Proposition~1.4 of \citet{benaim2005stochastic} (see also Theorem~5.3.3 of \citet{kushner2003stochastic}) that
\begin{align*}
    \lim_k \sup \Big\{ \big\| \sum_{j=k}^{k+l} \alpha_j w_j \big\|_\infty : l \geq 0 \text{ and } \sum_{j=k}^{k+l} \alpha_j \leq 8m + 1 \Big\} = 0 \, ,
\end{align*}
almost surely.
In particular, letting $\varepsilon \in (0, 1)$, there almost-surely exists $K$ (which depends on the realisation of the martingale noise) such that
\begin{align*}
    \sup \Big\{ \big\| \sum_{j=k}^{k+l} \alpha_j w_j \big\|_\infty : l \geq 0 \text{ and } \sum_{j=k}^{k+l} \alpha_j \leq 8m + 1 \Big\} < \varepsilon
\end{align*}
for all $k \geq K$, and further such that $\alpha_k < 1$ for all $k \geq K$.

Let us additionally take $\bar{M} \geq M$ such that $\delta \bar{M} \geq 4(8m + 1)$. Suppose that for some $k \geq K$, $\|\theta_k\|_\infty \geq \bar{M} + (8m + 1)$. Let $l$ be minimal such that $\sum_{j=k}^{k+l} \alpha_j > 8m$. Then we have
$\|\theta_{k+j} - \theta_k\|_\infty \leq 8m + 1$ for all $0 \leq j \leq l$, and so $\|\theta_{k+j}\|_\infty \geq \bar{M}$ for all $0 \leq j \leq l$. Further, if $\theta_k(x, i)$ satisfies $\theta_k(x, i) > \|\theta_k\|_\infty(1-\delta) + 2(8m+1)$, then we have
\begin{align*}
    & \phantom{=}\theta_{k+j}(x, i)\\
    & \geq  
    \theta_k(x, i) - (8m+1)\\
    & \geq 
    (1-\delta) \|\theta_k\|_\infty + (8m+1) \\
    & \geq
    (1-\delta)(\|\theta_{k+j}\|_\infty - (8m + 1)) + (8m+1) \\
    & =
    (1-\delta) \|\theta_{k+j}\|_\infty \, ,
\end{align*}
so $I^+_{k+j}(x, i)$ holds for all $0 \leq j \leq l$, and hence
\begin{align}\label{eq:reduce-theta}
    |\theta_{k+l+1}(x, i)| \leq \|\theta_k\|_\infty - \sum_{j=k}^{k+l} \alpha_j \times 1/(4m) + \|\sum_{j=k}^{k+l} \alpha_j w_j \|_\infty < \|\theta_k\|_\infty - 2 + \varepsilon < \|\theta_k\| - 1 \, .
\end{align}
Similarly, if $\theta_k(x, i) < -\|\theta_k\|_\infty(1-\delta) - 2(8m+1)$, $I^-_{k+j}(x, i)$ holds for all $0 \leq j \leq l$, and we reach the same conclusion as in Equation~\eqref{eq:reduce-theta}. Finally, if $|\theta_k(x, i)| \leq \|\theta_k\|_\infty(1-\delta) + 2(8m+1)$, then since $\delta \|\theta_k\|_\infty > \delta \bar{M}$, we have $|\theta_k(x, i)| \leq \|\theta_k\|_\infty - 2(8m+1)$, and hence $|\theta_{k+l+1}(x, i)| \leq \|\theta_k\|_\infty - (8m+1)$. Putting these components together, we have
\begin{align*}
    \|\theta_{k+l+1} \|_\infty < \|\theta_k\|_\infty + 1 \, , \text{ and } \max_{0 \leq j \leq l} \|\theta_{k+j}\|_\infty \leq \|\theta_k\|_\infty + (8m+1) \, ,
\end{align*}
as required to establish boundedness.

\subsection{Proof of Proposition~\ref{prop:general-lyapunov}}\label{sec:proof-general-lyapunov}

We first state and prove a useful lemma that allows us to compare QDP fixed points for different values of $\lambda$. Throughout this section, we will adopt the shorthand $\theta^\lambda$ for $\qdpfixedpointlambdaparams$.

\begin{lemma}\label{lemma:fixed-point-sensitivity}
    Let $\lambda, \lambda' \in [0,1]^{\mathcal{X} \times [m]}$. Then we have
    \begin{align*}
        \|\theta^{\lambda} - \theta^{\lambda'} \|_\infty \leq  C\| \lambda - \lambda' \|_\infty \, ,
    \end{align*}
    where $C$ is a constant depending only on the reward distributions of the MDP and $\gamma$.
\end{lemma}
\begin{proof}
    By the triangle inequality, we have
    \begin{align*}
        \| \theta^\lambda - \theta^{\lambda'} \|_\infty & \leq \| \theta^\lambda - \Pi^{\lambda'} \mathcal{T}^\pi \theta^{\lambda} \|_\infty + \| \Pi^{\lambda'} \mathcal{T}^\pi \theta^{\lambda} - \theta^{\lambda'} \|_\infty  \\
        & = \| \Pi^{\lambda} \mathcal{T}^\pi \theta^\lambda - \Pi^{\lambda'} \mathcal{T}^\pi \theta^{\lambda} \|_\infty + \| \Pi^{\lambda'} \mathcal{T}^\pi \theta^{\lambda} - \Pi^{\lambda'} \mathcal{T}^\pi \theta^{\lambda'} \|_\infty  \\
        & \leq \| (\Pi^{\lambda} - \Pi^{\lambda'}) \mathcal{T}^\pi \theta^\lambda \|_\infty + \gamma \| \theta^{\lambda} - \theta^{\lambda'} \|_\infty \\
        \implies \| \theta^\lambda - \theta^{\lambda'} \|_\infty  & \leq \frac{1}{1-\gamma} \| (\Pi^{\lambda} - \Pi^{\lambda'}) \mathcal{T}^\pi \theta^\lambda \|_\infty \, .
    \end{align*}
    Now we aim to bound $\|\theta^\lambda\|_\infty$, and hence the term on the right-hand side above. Note that in general for a mixture distribution $\nu = \sum_{i=1}^n p_i \nu_i$, we have $F^{-1}_\nu(\tau) \geq \min \{ F^{-1}_{\nu_i}(\tau) : i=1,\ldots,n\}$, since
    \begin{align*}
        \mathbb{P}_{Z \sim \nu}(Z \leq  \min \{ F^{-1}_{\nu_i}(\tau_n) : i=1,\ldots,n\} ) & = \sum_{i=1}^n p_i \mathbb{P}_{Z_i \sim \nu_i}(Z_i \leq \min \{ F^{-1}_{\nu_i}(\tau) : i=1,\ldots,n\}) \\
        & \leq \sum_{i=1}^n p_i \mathbb{P}_{Z_i \sim \nu_i}(Z_i \leq F^{-1}_{\nu_i}(\tau)) \\
        & \leq \tau \, .
    \end{align*}
    Thus, it follows that the $\nicefrac{1}{2m}$ quantile of $\mathcal{T}^\pi \theta^\lambda(x)$ is at least as great as
    \begin{align*}
        \min_{x} F_{\mathcal{R}^\pi(x)}^{-1}(\nicefrac{1}{2m}) - \gamma \|\theta^\lambda\|_\infty \, .
    \end{align*}
    By analogous reasoning, we obtain that $\bar{F}^{-1}_{(\mathcal{T}^\pi \theta^\lambda)(x)}(\nicefrac{2m-1}{2m})$ is no greater than
    \begin{align*}
        \max_{x} \bar{F}_{\mathcal{R}^\pi(x)}^{-1}(\nicefrac{2m-1}{2m}) + \gamma \|\theta^\lambda\|_\infty \, .
    \end{align*}
    From these facts, it follows that
    \begin{align*}
        \|\theta^\lambda\|_\infty \leq \frac{1}{1-\gamma} \max\left( |\min_{x} F_{\mathcal{R}^\pi(x)}^{-1}(\nicefrac{1}{2m})|, |\max_{x} \bar{F}_{\mathcal{R}^\pi(x)}^{-1}(\nicefrac{2m-1}{2m}) | \right) \, ,
    \end{align*}
    and hence
    \begin{align*}
        \| (\Pi^{\lambda} - \Pi^{\lambda'}) \mathcal{T}^\pi \theta^\lambda \|_\infty \leq C \| \lambda - \lambda' \|_\infty  \, ,
    \end{align*}
    as required for the statement of the result.
\end{proof}

We now turn to the proof of Proposition~\ref{prop:general-lyapunov}. First, we observe that the infimum over $\lambda$ in Equation~\eqref{eq:non-unique-lyapunov} is attained, since Lemma~\ref{lemma:fixed-point-sensitivity} establishes that $\lambda \mapsto \theta^{\lambda}$ is continuous (in fact Lipschitz), and $[0,1]^{\mathcal{X} \times [m]}$ is compact. We therefore have that $L$ is continuous, non-negative, and takes on the value $0$ only on the set of fixed points $\{ \theta^\lambda : \lambda \in [0,1]^{\mathcal{X} \times [m]}\}$.

For the decreasing property, let $(\vartheta_t)_{t \geq 0}$ be a solution to the differential inclusion in Equation~\eqref{eq:qtd-di}, and as in Definition~\ref{def:di-soln}, let $g: [0, \infty) \rightarrow \mathbb{R}^{\mathcal{X} \times [m]}$ satisfy
\begin{align}\label{eq:integral2}
    \vartheta_t = \int_0^t g_s\; \mathrm{d}s \, ,
\end{align}
with $g_t(x, i) \in H^\pi_{x,i}(\vartheta_t)$ for all $(x, i)$, and for almost all $t \geq 0$, where we have introduced the notation
\begin{align*}
     H^\pi_{x, i}(\theta) = [\tau_i - F_{(\mathcal{T}^\pi \theta)(x)}(\theta(x, i)), \tau_i - F_{(\mathcal{T}^\pi \theta)(x)}(\theta(x, i)-)] \, .
\end{align*}
As in the proof of Proposition~\ref{prop:lyapunov}, we will show that $L(\vartheta_t)$ is locally decreasing outside of the fixed point set, which is enough for the global decreasing property. Further, by continuity of $L(\vartheta_t)$, it is enough to show this property for almost all $t \geq 0$. We will therefore consider a value of $t \geq 0$ at which the above inclusion for $g_t$ holds.

Let $\overline{\lambda}$ attain the minimum in the definition of $L(\vartheta_t)$. Write $\theta^{\overline{\lambda}}$ for the corresponding fixed point for conciseness, and let $(x, i)$ be a $\overline{\lambda}$-\emph{argmax index with respect to $\vartheta_t$}; a state-particle pair achieving the maximum in the definition of the norm $\| \vartheta_t - \theta^{\overline{\lambda}}\|_\infty$.
First, we consider the cases where $H^\pi_{x, i}(\vartheta_t)$ is \emph{not} a singleton.
Now, if $0 \in H^\pi_{x, i}(\vartheta_t)$, then we have $(\Pi^\lambda \mathcal{T}^\pi \vartheta_t)(x, i) = \vartheta_t(x, i)$, and with the same logic as above, we have $\vartheta_t = \theta^{\overline{\lambda}}$, and hence $\vartheta_t$ is in the fixed point set, and $L(\vartheta_t)$ is constant. If $0 \not\in H^\pi_{x, i}(\vartheta_t)$, then as in the proof of Proposition~\ref{prop:lyapunov}, it can be shown that any element of $H^\pi_{x, i}(\vartheta_t)$ has the same sign as
\begin{align}\label{eq:same-sign}
    (\Pi^\lambda \mathcal{T}^\pi \vartheta_t)(x, i) - \vartheta_t(x, i) \, .
\end{align}
In the case of Proposition~\ref{prop:lyapunov}, continuity of the derivative then allowed us to deduce that $|\vartheta_t(x, i) - \theta^{\overline{\lambda}}(x, i)|$ is locally decreasing. Here, we require a related concept of continuity for the set-valued map $\theta \mapsto H^\pi_{x, i}(\theta)$, namely that it is upper semicontinuous (see, for example, \citeauthor{smirnov2002introduction}, \citeyear{smirnov2002introduction}); for a given $\theta \in \mathbb{R}^{\mathcal{X} \times [m]}$ and any given $\varepsilon > 0$, there exists $\delta > 0$ such that if $\|\theta' - \theta\|_\infty < \delta$, then $H^\pi_{x, i}(\theta') \subseteq \{ h + v : h \in H^\pi_{x, i}(\theta) \, , |v| < \varepsilon \}$. From this, it follows that any element of $H^\pi_{x, i}(\vartheta_{t+s})$, for sufficiently small positive $s$, has the same sign as the expression in Equation~\eqref{eq:same-sign}, and so from Equation~\eqref{eq:integral2}, we have that $|\vartheta_t(x, i) - \theta^{\overline{\lambda}}(x, i)|$ is locally decreasing, as required.

Now, when $H^\pi_{x, i}(\vartheta_t)$ \emph{is} a singleton, if it is non-zero, then by the same argument as in the proof of Proposition~\ref{prop:lyapunov}, the corresponding element has the same sign as
the expression in Equation~\eqref{eq:same-sign}, and so as above, we conclude that $|\vartheta_t(x, i) - \theta^{\overline{\lambda}}(x, i)|$ is locally decreasing.

Finally, the case where there exists an argmax index $(x, i)$ with $H^\pi_{x, i}(\vartheta_t) = \{0 \}$ requires more care, and we will need to reason about the effects of perturbing $\lambda$ to show that the Lyapunov function is decreasing. For some intuition as to what the problem is, if $H^\pi_{x, i}(\vartheta_{t+s}) = \{0 \}$ for small positive $s$, then the coordinate $\vartheta_{t+s}(x, i)$ is static, as it lies on the flat region of the CDF $F_{(\mathcal{T}^\pi \vartheta_{t + s})(x)}$ at level $\tau_i$, and so the distance $|\vartheta_{t+s}(x, i) - \theta^{\overline{\lambda}}(x, i)|$ is not decreasing. We explain how to deal with this case below.

\subsubsection{Perturbative Argument}

We introduce the notation $\indexset_0 \subseteq \mathcal{X} \times [m]$ for the set of $\overline{\lambda}$-argmax indices with respect to $\vartheta_t$. Assuming that $\| \vartheta_t - \theta^{\overline{\lambda}}\|_\infty$ is not locally decreasing, it must be locally constant (it cannot increase, by the arguments above). Now consider $s>0$ sufficiently small so that (i) no coordinates not in $\indexset_0$ can be a $\overline{\lambda}$-argmax index with respect to $\vartheta_{t+s}$, so that $\indexset$, the set of $\overline{\lambda}$-argmax indices with respect to $\vartheta_{t+s}$, satisfies $\indexset \subseteq \indexset_0$, (ii) all indices  $(x, i) \in \indexset$ satisfy $H^\pi_{x, i}(\vartheta_{t + u}) = \{ 0 \}$ for all $u \in [0,2s]$.

We will now demonstrate the existence of a parameter $\lambda' \in [0,1]^{\mathcal{X} \times [m]}$ such that $\|\vartheta_{t+s} - \theta^{\lambda'} \|_\infty < \|\vartheta_t - \theta^{\overline{\lambda}} \|_\infty$, which establishes the locally decreasing property of the Lyapnuov function, as required. To do so, we introduce a modification of the fixed point map $\lambda \mapsto \theta^\lambda$.
Letting $\mu \in \mathbb{R}^{\indexset}$, and defining $\overline{\lambda}[\mu] \in \mathbb{R}^{\mathcal{X} \times [m]}$ to be the replacement of the $\indexset$ coordinates of $\overline{\lambda}$ with the corresponding coordinates of $\mu$, we consider the map
\begin{align*}
   h_{\overline{\lambda}} : [0,1]^{\indexset} \rightarrow \mathbb{R}^{\indexset} \, , \quad  h(\mu) = \projmap \theta^{\overline{\lambda}[\mu]} \, ,
\end{align*}
where $\projmap : \mathbb{R}^{\mathcal{X} \times [m]} \rightarrow \mathbb{R}^{\indexset}$ extracts the $\indexset$ coordinates. 
At an intuitive level, this map allows us to study the effect of perturbing the $\indexset$ coordinates of $\overline{\lambda}$ on the corresponding coordinates of the fixed point.

\subsubsection{Case 1: $\overline{\lambda}_{\indexset}$ is in the Interior of $[0,1]^{\indexset}$}

We now first consider the case where $\overline{\lambda}_{\indexset}$, the $\indexset$ coordinates of $\overline{\lambda}$, lies in the interior of $[0,1]^{\indexset}$, that is $(0, 1)^{\indexset}$.
By Lemma~\ref{lemma:fixed-point-sensitivity}, $h_{\overline{\lambda}}$ is continuous, since it is the composition of the continuous maps $\mu \mapsto \overline{\lambda}[\mu]$, $\lambda \mapsto \theta^\lambda$, and $\theta \mapsto \projmap \theta$. It is also injective in a neighbourhood of $\overline{\lambda}_{\indexset}$. This can be seen by noting first that the fixed points $\theta^{\overline{\lambda}[\mu]}$ are distinct for distinct values of $\mu$ sufficiently close to $\overline{\lambda}_{\indexset}$; if $\mu \not= \mu'$ are each sufficiently close to $\overline{\lambda}_{\indexset}$, then we have
\begin{align*}
    \Pi^{\overline{\lambda}[\mu']} \mathcal{T}^\pi \theta^{\overline{\lambda}[\mu]} \not= \Pi^{\overline{\lambda}[\mu]} \mathcal{T}^\pi \theta^{\overline{\lambda}[\mu]} = \theta^{\overline{\lambda}[\mu]} \, ,
\end{align*}
where the inequality follows from the fact that since $\theta^{\overline{\lambda}[\mu]}$ is continuous in $\mu$, for $\mu$ sufficiently close to $\indexset$ there is a flat region of $F_{(\mathcal{T}^\pi \theta^{\overline{\lambda}[\mu]})(x)}$ at level $\tau_i$, for any $(x, i) \in \indexset$.
To complete the injectivity argument, we cannot have $\projmap \theta^{\overline{\lambda}[\mu']} = \projmap \theta^{\overline{\lambda}[\mu]}$ if $\theta^{\overline{\lambda}[\mu']} \not= \theta^{\overline{\lambda}[\mu]}$, as the contraction maps $\Pi^{\overline{\lambda}[\mu]} \mathcal{T}^\pi$ and $\Pi^{\overline{\lambda}[\mu]} \mathcal{T}^\pi$ are equal on coordinates not in $\indexset$, and these two maps would therefore have the same fixed point, a contradiction.

We may now appeal to the invariance of domain theorem \citep{brouwer1911beweis} to deduce that since $h_{\overline{\lambda}}$ is a continuous injective map between an open subset of $[0,1]^{\indexset}$ containing $\overline{\lambda}_{\indexset}$ (here we are using the assumption that $\overline{\lambda}_{\indexset}$ lies in the interior of $[0,1]^{\indexset}$) and the Euclidean space $\mathbb{R}^{\indexset}$ of equal dimension, it is an open map on this domain; that is, it maps open sets to open sets. Hence, we can perturb $\theta^{\overline{\lambda}}$ in the $\indexset$ coordinates in any direction we want by locally modifying the $\indexset$ coordinates of $\overline{\lambda}$. In particular, we can move all $\indexset$ coordinates of $\theta^{\overline{\lambda}}$ closer to those of $(\vartheta_{t+s}(x, i) : (x, i) \in \indexset)$. Let $\lambda' \in (0, 1)^{\mathcal{X} \times [m]}$ be such a modification of $\overline{\lambda}$, taken to be close enough to $\overline{\lambda}$ so that all coordinates outside $\indexset$ have sufficiently small perturbations so that they cannot be $\lambda'$-argmax indices with respect to $\vartheta_{t+s}$. We then have that $\|\vartheta_{t+s} - \theta^{\lambda'}\|_\infty < \| \vartheta_{t} - \theta^{\lambda}\|_\infty$, as required.

\subsubsection{Case 2: $\overline{\lambda}_{\indexset}$ is on the Boundary of $[0,1]^{\indexset}$}

In the more general case when $\overline{\lambda}_{\indexset}$ may lie on the boundary of $[0,1]^{\indexset}$, we can apply the same argument to an extension of the function $h_{\overline{\lambda}}$, by increasing its domain from $[0,1]^{\indexset}$ to an open neighbourhood of this domain in $\mathbb{R}^{\indexset}$. We define this extension simply by extending the definition of $\Pi^{\lambda}$ in Equation~\eqref{eq:proj-definition} to allow coordinates of $\lambda$ to lie outside the range $[0,1]$. We lose the non-expansiveness of $\Pi^{\lambda}$ (in $L^\infty$) under this extension, but if $\lambda_{\textrm{min}}, \lambda_{\textrm{max}}$ are the minimum and maximum coordinates of $\lambda$, respectively, it is easy verified (by modifying the proof of Proposition~\ref{prop:proj-nonexpansion}) that $\Pi^{\lambda}$ is $\max(1-\lambda_{\textrm{min}}, \lambda_{\textrm{max}})$-Lipschitz, and so if we extend the function to a domain where $\lambda_{\textrm{max}}, 1-\lambda_{\textrm{min}} \leq \gamma^{-1/2}$, the composition $\Pi^{\lambda} \mathcal{T}^\pi$ is a $\gamma^{1/2}$-contraction in $L^\infty$, and hence has a unique fixed point $\theta^\lambda$.

By the same arguments as above, the extended map $h_{\overline{\lambda}}$ is continuous and injective in a neighbourhood of $\overline{\lambda}_{\indexset}$ on this extended domain, and hence we may again apply the invariance of domain theorem to obtain that $h_{\overline{\lambda}_{\setminus \indexset}}$ is locally surjective around $\overline{\lambda}_{\indexset}$. However, since $\overline{\lambda}_{\indexset}$ lies on the boundary of the original domain, we must additionally check that we can perturb $\overline{\lambda}_{\indexset}$ to obtain $\mu$ in such a way that we obtain the desired perturbation of $\theta^{\overline{\lambda}}$, without the parameters $\mu$ leaving the set $[0,1]^{\indexset}$. To do this, we first rule out $\overline{\lambda}_{\indexset}$ lying on certain parts of the boundary.

\begin{lemma}\label{lem:boundary}
    If $(x, i) \in \indexset$ and $\vartheta_{t+s}(x, i) < \theta^{\overline{\lambda}}(x, i)$, then $\overline{\lambda}(x, i) > 0$. Similarly, if $\vartheta_{t+s}(x, i) > \theta^{\overline{\lambda}}(x, i)$, then $\overline{\lambda}(x, i) < 1$.
\end{lemma}
\begin{proof}
    We prove the claim when $\vartheta_{t+s}(x, i) < \theta^{\overline{\lambda}}(x, i)$; the other case follows analogously. If $\overline{\lambda}(x, i) = 0$, then since $\vartheta_{t+s}(x, i)$ corresponds to the flat region at level $\tau_i$ of the CDF $F_{(\mathcal{T}^\pi \vartheta_{t+s})(x)}$, we must have $(\Pi^{\overline{\lambda}} \mathcal{T}^\pi \vartheta_{t+s})(x, i) \leq \vartheta_{t+s}(x, i)$ since $\overline{\lambda}(x, i) = 0$, and so the chosen quantile at level $\tau_i$ by the projection $\Pi^{\overline{\lambda}}$ is the left-most point of this flat region. We therefore have
    \begin{align*}
        \overline{w}_\infty(\Pi^{\overline{\lambda}} \mathcal{T}^\pi \vartheta_t, \theta^{\overline{\lambda}}) \geq |  (\Pi^{\overline{\lambda}} \mathcal{T}^\pi \vartheta_t)(x, i) - \theta^{\overline{\lambda}}(x, i) | \geq |\vartheta_t(x, i) - \theta^{\overline{\lambda}}(x, i) | = \overline{w}_\infty(\vartheta_t, \theta^{\overline{\lambda}}) \, ,
    \end{align*}
    contradicting contractivity of $\Pi^{\overline{\lambda}} \mathcal{T}^\pi$ around $\theta^{\overline{\lambda}}$.
\end{proof}

We write $v = \text{sign}((\vartheta_{t+s})_\indexset - \theta^{\overline{\lambda}}_\indexset) \in \mathbb{R}^\indexset$, where the sign mapping is applied elementwise, and introduce the notation $\cone(v) = \{\alpha \odot v :  \alpha \in \mathbb{R}^n_{> 0} \}$ for the (open) orthant containing the vector $v$. We are therefore seeking a perturbation $\mu$ of $\overline{\lambda}_\indexset$ such that $\theta^{\overline{\lambda}[\mu]}_J$ lies in a direction in $\cone(v)$ from $\theta^{\overline{\lambda}}_\indexset$, and further such that the perturbation to $\theta^{\overline{\lambda}}$ is sufficiently small that no index that was not an argmax in $\|\vartheta_{t+s} - \theta^{\overline{\lambda}}\|_\infty$ can become one in $\|\vartheta_{t+s} - \theta^{\overline{\lambda}[\mu]}\|_\infty$; under these conditions, we have $\|\vartheta_{t+s} - \theta^{\overline{\lambda}[\mu]}\|_\infty < \|\vartheta_{t+s} - \theta^{\overline{\lambda}}\|_\infty$, as required. Lemma~\ref{lem:boundary} then guarantees that a (sufficiently small) perturbation of $\overline{\lambda}_\indexset$ in any direction in $\cone(v)$ remains within $[0,1]^\indexset$, so it is sufficient to show that a perturbation in such a direction achieves the desired perturbation of $\theta^{\overline{\lambda}}$.

\emph{Differentiability.}
Now, if the extended map $\lambda \mapsto \theta^\lambda$ is differentiable at $\overline{\lambda}$, then differentiating through the fixed-point equation $\theta^\lambda = G(\lambda, \theta^\lambda)$ (where we write $G(\lambda, \theta) = \Pi^\lambda \mathcal{T}^\pi \theta$ for conciseness) yields
\begin{align*}
    \nabla_\lambda \theta^\lambda = \partial_\lambda G(\lambda, \theta^\lambda) + \partial_\theta G(\lambda, \theta^\lambda) \nabla_\lambda \theta^\lambda \, ;
\end{align*}
differentiability of $G$ in $\theta$ results from differentiability of the map $\lambda \mapsto G(\lambda, \theta^\lambda)$, and continuous differentiability of $G$ in $\lambda$. Since $\theta \mapsto G(\lambda, \theta)$ is contractive in $L^\infty$ with factor $\gamma^{1/2}$ (on the extended domain), and by coordinatewise monotonicity of $\theta \mapsto G(\lambda, \theta)$, it follows that $\partial_\theta G(\lambda, \theta^\lambda)$ is non-negative and strictly substochastic, with row $L^1$ norms bounded by $\gamma^{1/2}$, the contraction factor for the extended set of contraction mappings. We remark as a point of independent interest that this is a kind of Bellman equation for $\nabla_\lambda \theta^\lambda$, with $\partial_\theta G(\lambda, \theta^\lambda)$ taking the role of the transition matrix, and $\partial_\lambda G(\lambda, \theta^\lambda)$ taking the role of a collection of cumulants; in fact, the structure of $\partial_\theta G(\lambda, \theta^\lambda)$ coincides with the local quantile back-up diagrams described in Example~\ref{ex:fixed-points}.
We therefore have
\begin{align*}
    \nabla_\lambda \theta^\lambda = (I - \partial_\theta G(\lambda, \theta^\lambda))^{-1} \partial_\lambda G(\lambda, \theta^\lambda) \, .
\end{align*}
By extracting the principal submatrix on the $\indexset$ coordinates, we obtain a derivative for $h_{\overline{\lambda}}(\overline{\lambda}_\indexset)$. The following lemma is useful in reasoning about the structure of this principal submatrix.

\begin{lemma}\label{lem:principal-submatrix}
    Let $Q_1 \in \mathbb{R}^{n \times n}$ be strictly substochastic, and let $K \subseteq [n]$. Then the principal submatrix on the $K$ coordinates of $(I - Q_1)^{-1}$ can be expressed as $(I - Q_2)^{-1}$, with $Q_2 \in \mathbb{R}^{K \times K}$ strictly substochastic.
\end{lemma}
\begin{proof}
    We interpret $Q_1$ as the transition matrix of a Markov chain $(Z_t)_{t \geq 0}$ that includes a non-zero probability of termination at each state. Each row of the matrix $(I - Q_1)^{-1}$ is then the pre-termination visitation measure associated with a particular initial state in the Markov chain. Now let $Q_2$ be the strictly substochastic matrix defined by
    \begin{align*}
        Q_2(z_1, z_2) = \mathbb{P}((Z_t)_{t \geq 0}& \text{ does not terminate before returning to } K, 
        \\&\text{ first state on return is } z_2 \mid Z_0 = z_1) \, .
    \end{align*}
    By construction, the pre-termination visitation distribution $(I - Q_2)^{-1}$ is identical to the principal submatrix of $(I - Q_1)^{-1}$ on the $K$ coordinates, as required.
\end{proof}

From Lemma~\ref{lem:principal-submatrix}, we therefore obtain that $\nabla h_{\overline{\lambda}}(\overline{\lambda}_\indexset)$ has the form
\begin{align*}
    \nabla h_{\overline{\lambda}}(\overline{\lambda}_\indexset) = (I - Q)^{-1} D \, ,
\end{align*}
with $D \in \mathbb{R}^{\indexset \times \indexset}$ diagonal, with positive elements on the diagonal (from monotonicity of $\lambda \mapsto G(\lambda, \theta)$), with $Q \in \mathbb{R}^{\indexset \times \indexset}$ strictly substochastic. The derivative is therefore invertible, and we obtain the derivative of the inverse of the form
\begin{align*}
    \nabla h^{-1}_{\overline{\lambda}}(\theta^{\lambda}_\indexset) = D^{-1} (I - Q) \, .
\end{align*}
From strict substochasticity of $Q$, and since $v \in \{\pm 1\}^\indexset$, it follows that for the desired perturbation direction $v$, we have
\begin{align*}
    \nabla h^{-1}_{\overline{\lambda}}(\theta^\lambda_\indexset) v \signequals v \, ,
\end{align*}
and so $\nabla h^{-1}_{\overline{\lambda}}(\theta^\lambda_\indexset) v \in \cone(v)$,
where the equality of signs applies elementwise. Therefore, a perturbation of $\overline{\lambda}_{\indexset}$ in a direction in $\cone(v)$ is achieved by a sufficiently small perturbation of $\overline{\lambda}_\indexset$ in a direction in $\cone(v)$, as required.

\emph{Non-differentiability.}
If $\lambda \mapsto \theta^\lambda$ is not differentiable at $\overline{\lambda}$, we instead use techniques from non-smooth analysis to complete the argument. 
First, since $\lambda \mapsto \theta^\lambda$ is Lipschitz (by Lemma~\ref{lemma:fixed-point-sensitivity}), it is differentiable almost everywhere by Rademacher's theorem \citep{rademacher1919partielle}. By adapting the argument made by \citet[Lemma~3]{clarke1976inverse}, by Fubini's theorem, for almost all $\lambda_{\setminus \indexset} \in \mathbb{R}^{\mathcal{X} \times [m] \setminus \indexset}$, the map $\lambda \mapsto \theta^\lambda$ is differentiable at $(\lambda_{\setminus \indexset}, \mu)$ for almost all $\mu$ with $(\lambda_{\setminus \indexset}, \mu)$ in the extended domain. The map $(\lambda_{\setminus \indexset}, \mu) \mapsto (\lambda_{\setminus \indexset}, h_{\lambda_{\setminus \indexset}}(\mu))$ is Lipschitz and locally injective around $\overline{\lambda}$, and hence maps sufficiently small open neighbourhoods of $\overline{\lambda}$ to open neighbourhoods of $(\overline{\lambda}_{\setminus \indexset}, h_{\overline{\lambda}_{\setminus \indexset}}(\overline{\lambda}_\indexset))$. Further, since each $h_{\lambda_{\setminus \indexset}}$ is Lipschitz, and so absolutely continuous, the inverse map $h^{-1}_{\lambda_{\setminus \indexset}}$ is almost-everywhere differentiable within such a neighbourhood. Following the  analysis of the differentiable case, we therefore deduce
\begin{align*}
    \nabla h^{-1}_{\lambda_{\setminus \indexset}}(\theta) v \signequals v 
\end{align*}
for almost all $\lambda_{\setminus \indexset}$ in a ball $B$ around $\overline{\lambda}_{\setminus \indexset}$, and (for each such $\lambda_{\setminus \indexset}$) for almost all $\theta$ in the $L^\infty$ ball $B'$ with centre $h_{\overline{\lambda}_{\setminus \indexset}}(\overline{\lambda}_\indexset)$ and radius $\rho$, for some radius $\rho > 0$. We further take $B$ and $B'$ to be of small enough radii so that this directional derivative is bounded on this set, so that $h_{\lambda}^{-1}$ is locally Lipschitz on $B'$ for each $\lambda_{\setminus \indexset} \in B$ (and hence absolutely continuous), and so that for any $\theta \in B'$, we have $\text{sign}(\theta - (\vartheta_{t+s})_\indexset) = \text{sign}(\theta^{\overline{\lambda}}_\indexset - (\vartheta_{t+s})_\indexset)$, and so that for no $\mu$ in the preimage of $B'$ under $h_{\lambda_{\setminus \indexset}}$ can have that $\|\theta^{(\lambda_{\setminus \indexset}, \mu)} - \vartheta_{t+s} \|_\infty$ has new argmax coordinates outside of $\indexset$.

Let us consider $\tilde{\lambda} \in B$ at which the almost-everywhere differentiability condition holds. By applying the same argument with Fubini's theorem, for almost all $\bar{\theta}$ in $B(h_{\tilde{\lambda}}(\overline{\lambda}_\indexset), \rho/4)$, the inverse $h_{\tilde{\lambda}}^{-1}$ is differentiable almost everywhere on $\{\bar{\theta} + u v : u \in [0,\rho/2] \}$.

Now, defining $\mu_\tau = h^{-1}_{\tilde{\lambda}}(\bar{\theta} + \tau v)$ for $\tau \in [0,\rho/2]$, we have
\begin{align*}
    \frac{\mathrm{d}}{\mathrm{d}\tau} \mu_\tau = \nabla h^{-1}_{\tilde{\lambda}}(\bar{\theta} + \tau v) v
\end{align*}
for almost all $\tau$, and by absolute continuity of $h^{-1}_{\tilde{\lambda}}$, it follows that 
\begin{align*}
    \mu_{\rho/2}= \mu_0 + \int_0^{\rho/2} \frac{\mathrm{d}}{\mathrm{d}\tau}\mu_\tau \; \mathrm{d}\tau \, .
\end{align*}
Hence, $\mu_\varepsilon - \mu_0 \in \cone(v)$, and by construction $h_{\tilde{\lambda}}(\mu_{\rho/2}) = \bar{\theta} + \rho v / 2$. By continuity of $\lambda \mapsto \theta^\lambda$ and its inverse, and since $\tilde{\lambda}$ and $\bar{\theta}$ can be chosen above to be arbitrarily close to $\overline{\lambda}_{\setminus \indexset}$ and $h_{\overline{\lambda}_{\setminus \indexset}}(\overline{\lambda}_{\indexset})$ respectively, we may consider a sequence of these parameters converging to $\overline{\lambda}_{\setminus \indexset}$ and $h_{\overline{\lambda}_{\setminus \indexset}}(\overline{\lambda}_{\indexset})$, such that the values of $\mu_{\rho/2}$ as constructed above also converge (by compactness), and thus conclude the existence of $\bar{\mu}_{\rho/2}$ such that $\bar{\mu}_{\rho/2} - \overline{\lambda}_\indexset \in \overline{\cone(v)}$, and $h_{\overline{\lambda}}(\bar{\mu}_{\rho/2}) = h_{\overline{\lambda}}(\overline{\lambda}_{\indexset}) + \rho v/2$, as required.

\subsection{Proof of Proposition~\ref{prop:fixed-point-quality-1}}

We begin with the observation that for any return-distribution function $\eta \in \mathscr{P}([V_\textsc{min}, V_\textsc{max}])$, for the projection $\Pi^{\lambda}$ onto $\mathscr{F}_{\text{Q}, m}$ (for any $\lambda \in [0,1]^{\mathcal{X} \times [m]}$), we have
\begin{align*}
    \bar{w}_1(\Pi^\lambda \eta, \eta) \leq \frac{V_\textsc{max} - V_{\textsc{min}}}{2m} \, .
\end{align*}
Using this observation, we have
\begin{align*}
    \bar{w}_1(\qdpfixedpointlambda, \eta^\pi)
    & \overset{(a)}{\leq} \bar{w}_1(\qdpfixedpointlambda, \mathcal{T}^\pi \qdpfixedpointlambda) + \bar{w}_1(\mathcal{T}^\pi \qdpfixedpointlambda, \eta^\pi) \\
    & \overset{(b)}{=} \bar{w}_1(\Pi^\lambda \mathcal{T}^\pi \qdpfixedpointlambda, \mathcal{T}^\pi \qdpfixedpointlambda) + \bar{w}_1(\mathcal{T}^\pi \eta^\pi, \mathcal{T}^\pi \eta^\pi) \\
    & \overset{(c)}{\leq}
    \frac{V_\textsc{max} - V_{\textsc{min}}}{2m} + \gamma \bar{w}_1(\qdpfixedpointlambda, \eta^\pi) \, .
\end{align*}
Here, (a) follows from the triangle inequality, (b) follows as $\qdpfixedpointlambda$, $\eta^\pi$ are fixed points of $\Pi^\lambda \mathcal{T}^\pi$, $\mathcal{T}^\pi$, respectively, and (c) follows from the application of the inequality at the beginning of the proof and contractivity of $\mathcal{T}^\pi$. Rearranging then gives the desired result.

\subsection{Proof of Proposition~\ref{prop:improved-fp-bound}}

From the assumptions of the proposition, we have $\qdpfixedpointlambda = (\Pi^\lambda \mathcal{T}^\pi)^k \eta^\pi$. Then observe that following the argument for the proof of Proposition~\ref{prop:fixed-point-quality-1}, we have, for any $l \in \{1,\ldots,k\}$,
\begin{align*}
    \bar{w}_1((\Pi^\lambda \mathcal{T}^\pi)^l \eta^\pi, \eta^\pi) \leq \gamma \bar{w}_1((\Pi^\lambda \mathcal{T}^\pi)^{l-1} \eta^\pi, \eta^\pi) + \frac{V_\textsc{max} - V_\textsc{min}}{2m} \, .
\end{align*}
Chaining these inequalities yields the required statement.

\section{Implementations of Quantile Dynamic Programming}\label{sec:qdp-implementations}

Here, we describe two concrete implementations of QDP, which may be of independent interest to the reader. Algorithm~\ref{alg:qdp-discrete} \citep{bdr2022} describes an implementation when the reward distributions are available as input to the algorithm as a list of outcomes and probabilities.

\begin{algorithm}
    \begin{algorithmic}[1]
        \REQUIRE Quantile estimates $((\theta(x, i)_{i=1}^m : x \in \mathcal{X})$, \\
        Transition and reward probabilities $(P^\pi(x', r\mid x) : x, x' \in \mathcal{X})$, \\
        Interpolation parameters $\lambda \in [0, 1]^{\mathcal{X} \times [m]}$.
        \FOR{$x \in \mathcal{X}$}
            \STATE Set \texttt{Targets} as empty list \COMMENT{List of outcome/probability pairs}
            \FOR{$x' \in \mathcal{X}$}
                \FOR{$r \in \mathscr{R}$}
                    \FOR{$j=1,\ldots,m$}
                        \STATE Append $(r + \gamma \theta(x', j), P^\pi(x', r|x)/m)$ to \texttt{Targets}
                    \ENDFOR
                \ENDFOR
            \ENDFOR
            \STATE Sort \texttt{Targets} ascending according to outcomes.
            \FOR{$i =1,\ldots,m$}
                \STATE Find minimal outcome $q'$ such that cumulative probability is $\geq \nicefrac{2i-1}{2m}$.\alglinelabel{algline:inv-cdf}
                \STATE Set $\theta'(x, i) \leftarrow q'$. \alglinelabel{algline:inv-cdf-assign}
            \ENDFOR
        \ENDFOR
        \STATE \textbf{return} $((\theta'(x, i)_{i=1}^m : x \in \mathcal{X})$
    \end{algorithmic}
    \caption{Quantile dynamic programming (finitely-supported rewards)}
    \label{alg:qdp-discrete}
\end{algorithm}

Algorithm~\ref{alg:qdp-cont} makes use of a root-finding subroutine (such as \texttt{scipy.optimize.root\_scalar}), and can be used when the CDFs of the reward distributions are available as input, and can be queried at individual points. A common use case for this implementation is the case of Gaussian rewards. Note that the root-finding subroutine is called on a monotonic scalar function, and therefore strong guarantees can be given on the approximate solution returned when the reward CDFs of the MDP are continuous. Nevertheless, note that Algorithm~\ref{alg:qdp-cont} does not exactly implement the operator $\Pi^\lambda \mathcal{T}^\pi$ due to this root-finding approximation error. For simplicity, we present the algorithm in the case where the reward and next state in a transition are conditionally independent given the current state, though the algorithm can be straightforwardly extended to the general case, by working with CDFs of reward distributions conditioned on the next state.

\begin{algorithm}
    \begin{algorithmic}[1]
        \REQUIRE Quantile estimates $((\theta(x, i)_{i=1}^m : x \in \mathcal{X})$, \\
        Transition probabilities $(P^\pi(x' \mid x) : x, x' \in \mathcal{X})$, \\
        Reward CDFs $(F_{\mathcal{R}^\pi(x)} : x \in \mathcal{X})$.
        \FOR{$x \in \mathcal{X}$}
            \STATE Construct function
            \begin{align*}
                \phi_x : t \mapsto \sum_{x' \in \mathcal{X}} P^\pi(x'|x) \sum_{j=1}^m F_{\mathcal{R}^\pi(x)}(t - \gamma \theta(x', j))
            \end{align*}
            \FOR{$i=1,\ldots,m$}
                \STATE Use a scalar root-finding subroutine to find $\theta'(x, i)$ approximately satisfying
            \begin{align*}
                \phi_x(\theta'(x, i)) = \tau_i
            \end{align*}    
            \ENDFOR
        \ENDFOR
        \STATE \textbf{return} $((\theta'(x, i)_{i=1}^m : x \in \mathcal{X})$
    \end{algorithmic}
    \caption{Quantile dynamic programming (reward CDFs)}
    \label{alg:qdp-cont}
\end{algorithm}

\section{Convergence of Asynchronous QTD Updates}\label{sec:async}

Here, we describe the key considerations in extending our analysis to a proof of convergence for asynchronous versions of QTD; our discussion follows the approach of \citet{perkins2013asynchronous}.

\emph{Step size restrictions.} Typically, more restrictive assumptions on step sizes, beyond the Robbins-Monro conditions, are required for asynchronous convergence guarantees. See, for example, Assumption~A2 of \citet{perkins2013asynchronous}; note that the typical Robbins-Monro step size schedule of $\alpha_k \propto 1/k^\rho$ for $\rho \in (1/2, 1]$ satisfies these requirements.

\emph{Conditions on the sequence of states $(X_k)_{k \geq 0}$ to be updated.} Additionally, different states are required to be updated ``comparably often''; assuming that $(X_k)_{k \geq 0}$ forms an aperiodic irreducible time-homogeneous Markov chain is sufficient, and this conditions holds when either (i) $\pi$ generates such a Markov chain over the state space of the MDP of interest, or (ii) when the states to be updated are sampled i.i.d.\ from a fixed distribution supported on the entirety of the state space, amongst other settings. See Assumption~A4 of \citet{perkins2013asynchronous} for further details.

\emph{Modified differential inclusion.} The QTD differential inclusion in Equation~\eqref{eq:qtd-di} must be broadened to account for the possibility of different states being updated with different frequencies, leading to a differential inclusion of the form
\begin{align*}
    \partial_t \vartheta_t(x, i) \in \{ \omega h : \omega \in (\delta, 1] \, , h \in  H^\pi_{x, i}(\vartheta_t) \} \, ,
\end{align*}
where $\delta$ represents a minimum relative update frequency for the state $x$, derived from the conditions on $(X_k)_{k \geq 0}$ described above. Because of the structure of the Lyapunov function for the QTD DI in Equation~\eqref{eq:non-unique-lyapunov}, it is readily verified that this remains a valid Lyapunov function for this broader differential inclusion, for the same invariant set of QDP fixed points.

\vskip 0.2in
\bibliography{main}

\end{document}